\documentclass[pmlr]{jmlr}

\usepackage{amsmath,amssymb,graphicx,url}
\usepackage{booktabs}
\usepackage{microtype}
\usepackage{hyperref}
\hypersetup{
  hypertexnames=false,
  pdftitle={Conformal Calibration for Multi-Modal Regression with Missing Modalities},
  pdfauthor={Ilia Azizi},
  pdfsubject={Modality-aware conformal calibration for regression},
  pdfkeywords={conformal prediction, multi-modal regression, missing modalities}
}
\expandafter\def\expandafter\UrlBreaks\expandafter{\UrlBreaks\do\-}
\usepackage{placeins}
\usepackage{tikz}
\usetikzlibrary{arrows.meta,positioning}
\usepackage{algorithm}
\usepackage{algcompatible}
\usepackage{dsfont}

\AtBeginDocument{}
\newcommand{\appref}[1]{\hyperref[#1]{Appendix~\ref*{#1}}}

\jmlrvolume{329}
\jmlryear{2026}
\jmlrworkshop{Conformal and Probabilistic Prediction with Applications}
\jmlrproceedings{PMLR}{Proceedings of Machine Learning Research}

\title[Conformal Calibration for Multi-Modal Regression with Missing Modalities]{Conformal Calibration for Multi-Modal Regression\\[0.2cm]with Missing Modalities}

\author{%
  \Name{Ilia Azizi}\Email{ilia.azizi@unil.ch}\\
  \addr{Department of Operations, HEC Lausanne, University of Lausanne, Switzerland}\\
  \addr{BegooAI, Switzerland}
  }

\editor{Ernst Ahlberg, Ulf Johansson, Henrik Bostr\"om, Alberto Carlevaro, Johan Hallberg Szabadv\'ary and Lars Carlsson}

\begin{document}

\maketitle

\begin{abstract}
Prediction intervals for multi-modal regression with tabular variables, text, images, or other input sources are difficult to calibrate when those sources disagree or one is missing. A single global quantile averages these regimes together instead of calibrating to the modality pattern observed at test time. We address this through a modality-aware conformal calibration layer. The layer trains or reuses one predictor per modality, computes a disagreement score from their predictions, and uses that score in split conformal calibration under a strict split protocol. We use the score in two complementary ways. First, a continuous disagreement-scaled method reallocates interval width across examples while preserving the usual marginal split-conformal guarantee. Second, a Mondrian (stratified) method calibrates within groups defined by disagreement or modality availability fixed before calibration, giving group guarantees under joint exchangeability of the calibration and test examples. Across four multi-modal datasets, the disagreement-scaled layer matches or improves the marginal conformal baseline in 59 of 60 paired runs for interval continuous ranked probability score (CRPS) and in 52 of 60 for interval width, while keeping empirical coverage near the 95\% target. In stress tests with missing modalities, mask-matched recalibration recovers up to 19.5 percentage points of coverage in the hardest fixed-mask regime. The result is a simple, model-agnostic reliability layer for multi-modal regression systems. A project page is available at \url{https://unco3892.github.io/modality-aware-conformal}.
\end{abstract}


\section{Introduction}
\label{sec:intro}

\begin{figure}[!htbp]
\centering
\definecolor{tabcol}{RGB}{38,108,196}
\definecolor{txtcol}{RGB}{33,150,80}
\definecolor{imgcol}{RGB}{226,138,22}
\definecolor{basecol}{RGB}{20,138,138}
\definecolor{calcol}{RGB}{124,55,160}
\providecommand{\icTab}[2]{\begin{scope}[shift={(#1,#2)}]
  \draw[tabcol,line width=0.5pt,fill=tabcol!10] (-0.20,-0.16) rectangle (0.20,0.16);
  \fill[tabcol!35] (-0.20,0.055) rectangle (0.20,0.16);
  \draw[tabcol,line width=0.4pt] (-0.20,0.055)--(0.20,0.055);
  \draw[tabcol,line width=0.4pt] (-0.20,-0.05)--(0.20,-0.05);
  \draw[tabcol,line width=0.4pt] (-0.067,-0.16)--(-0.067,0.16);
  \draw[tabcol,line width=0.4pt] (0.067,-0.16)--(0.067,0.16);
\end{scope}}
\providecommand{\icTxt}[2]{\begin{scope}[shift={(#1,#2)}]
  \draw[txtcol,line width=0.5pt,fill=txtcol!8] (-0.165,-0.18) rectangle (0.165,0.18);
  \draw[txtcol,line width=0.55pt,line cap=round] (-0.10,0.10)--(0.10,0.10);
  \draw[txtcol,line width=0.55pt,line cap=round] (-0.10,0.025)--(0.10,0.025);
  \draw[txtcol,line width=0.55pt,line cap=round] (-0.10,-0.05)--(0.10,-0.05);
  \draw[txtcol,line width=0.55pt,line cap=round] (-0.10,-0.125)--(0.03,-0.125);
\end{scope}}
\providecommand{\icImg}[2]{\begin{scope}[shift={(#1,#2)}]
  \draw[imgcol,line width=0.5pt,fill=imgcol!10] (-0.20,-0.16) rectangle (0.20,0.16);
  \fill[imgcol!80] (0.095,0.075) circle (0.036);
  \draw[imgcol,line width=0.5pt,line join=round,fill=imgcol!30]
       (-0.20,-0.16)--(-0.075,0.02)--(0.02,-0.055)--(0.13,0.055)--(0.20,-0.02)--(0.20,-0.16)--cycle;
\end{scope}}
\providecommand{\icImgOff}[2]{\begin{scope}[shift={(#1,#2)}]
  \draw[black!40,line width=0.5pt,fill=black!6] (-0.20,-0.16) rectangle (0.20,0.16);
  \draw[black!38,line width=0.5pt] (0.095,0.075) circle (0.036);
  \draw[black!38,line width=0.5pt,line join=round]
       (-0.20,-0.16)--(-0.075,0.02)--(0.02,-0.055)--(0.13,0.055)--(0.20,-0.02);
  \draw[red!62,line width=0.8pt,line cap=round] (-0.225,0.185)--(0.225,-0.185);
  \draw[red!62,line width=0.8pt,line cap=round] (-0.225,-0.185)--(0.225,0.185);
\end{scope}}
\resizebox{0.98\linewidth}{!}{%
\begin{tikzpicture}[
  font=\small,
  >={Stealth[length=2.2mm,width=1.5mm]},
  line cap=round, line join=round,
  bandframe/.style={draw=black!45, dash pattern=on 2.6pt off 2.4pt,
                    line width=0.7pt, rounded corners=8pt},
  bandtitle/.style={fill=white, inner sep=2.4pt, font=\scriptsize\bfseries,
                    text=black!72},
  mod/.style  ={rectangle, rounded corners=2.5pt, line width=0.6pt, font=\scriptsize,
               text width=19.5mm, minimum height=7mm, align=left, inner sep=2.5pt},
  pred/.style ={rectangle, rounded corners=2.5pt, line width=0.6pt, font=\footnotesize,
               minimum width=8mm, minimum height=7mm, align=center, inner sep=2.5pt},
  dbox/.style ={rectangle, draw=calcol!70, line width=0.7pt, rounded corners=3pt,
               text width=22mm, fill=white, align=center, inner sep=4pt},
  calbox/.style={rectangle, draw=calcol!70, line width=0.7pt, rounded corners=3pt,
               text width=30mm, fill=white, align=left, inner sep=4pt},
  ext/.style  ={rectangle, draw=basecol, dash pattern=on 2.2pt off 1.8pt,
               line width=0.8pt, rounded corners=3pt,
               text width=21mm, fill=basecol!6, align=center, inner sep=4pt},
  fin/.style  ={rectangle, draw=calcol!70, line width=0.7pt, rounded corners=3pt,
               text width=26mm, fill=calcol!10, align=center, inner sep=3.5pt},
  arr/.style  ={->, draw=black!75, line width=0.6pt},
  lin/.style  ={draw=black!75, line width=0.6pt},
  rowlab/.style={font=\footnotesize, text=black!78, align=left},
  ptitle/.style={font=\footnotesize\bfseries, text=black!85},
  psub/.style ={font=\footnotesize, text=calcol!75},
  basebar/.style={line width=1.6pt, draw=basecol!85, line cap=round},
  globbar/.style={line width=1.4pt, draw=black!42, line cap=round},
  calbar/.style ={line width=2.8pt, draw=calcol!80, line cap=round},
  endt/.style   ={line width=1.0pt, line cap=round},
  flankL/.style={font=\scriptsize, anchor=east},
  flankR/.style={font=\scriptsize, anchor=west},
  lead/.style ={draw=black!30, line width=0.35pt}
]
\def\CX{8.68}    
\draw[bandframe] (-0.45,2.85) rectangle (14.20,6.25);
\node[bandtitle,anchor=west] at (0.05,6.25) {The calibration layer};
\draw[bandframe] (-0.45,-2.85) rectangle (14.20,2.20);
\node[bandtitle,anchor=west] at (0.05,2.20)
{Worked example \textnormal{for $\log(\mathrm{rent})$ at a $95\%$ target,
       where a global $\hat q$ does not adapt across regimes}};
\node[mod,draw=tabcol,fill=tabcol!7,anchor=west] (m1) at (-0.30,5.55) {\hspace{5.0mm}Table $x^{(1)}$};
\icTab{0.00}{5.55}
\node[mod,draw=txtcol,fill=txtcol!7,anchor=west] (m2) at (-0.30,4.55) {\hspace{5.0mm}Text $x^{(2)}$};
\icTxt{0.00}{4.55}
\node[mod,draw=imgcol,fill=imgcol!7,anchor=west] (m3) at (-0.30,3.55) {\hspace{5.0mm}Image $x^{(3)}$};
\icImg{0.00}{3.55}
\node[pred,draw=tabcol,fill=tabcol!6] (g1) at (2.87,5.55) {$g_1$};
\node[pred,draw=txtcol,fill=txtcol!6] (g2) at (2.87,4.55) {$g_2$};
\node[pred,draw=imgcol,fill=imgcol!6] (g3) at (2.87,3.55) {$g_3$};
\foreach \i in {1,2,3} \draw[arr] (m\i) -- (g\i);
\node[ext,anchor=south] (b) at (5.15,0 |- m3.south)
  {\scriptsize Base $[\ell,u]$\\[0.5pt]
   \scriptsize\itshape quantile/point};
\coordinate (bt) at ([yshift=3.5mm]b.north);
\node[dbox,anchor=south] (d) at (5.15,0 |- bt)
  {\scriptsize\textbf{Disagreement}\\[2pt]
   \scriptsize $d(x)=\operatorname{sd}_k(\hat y_k)$};
\coordinate (gx) at ([xshift=-3.2mm]d.west);
\draw[lin,rounded corners=1.2pt] (g1.east) -- (gx |- g1) -- (gx |- d);
\draw[lin,rounded corners=1.2pt] (g2.east) -- (gx |- g2) -- (gx |- d);
\draw[lin,rounded corners=1.2pt] (g3.east) -- (gx |- g3) -- (gx |- d);
\draw[arr] (gx |- d) -- (d.west);
\node[calbox,anchor=south] (cal) at ({\CX},0 |- b.south)
  {\scriptsize\textbf{Scaled CQR}\\[1.5pt]
   \scriptsize $a_\gamma(x)=\sqrt{1+\gamma\,d(x)^2}$\\[1.5pt]
   \scriptsize $s_i=e_i^+/a_\gamma(X_i)$\\[7pt]
   \scriptsize\textbf{Mondrian}\\[1.5pt]
   \scriptsize $\hat q_h$ per stratum $h$\\[1.5pt]
   \scriptsize $h$: $d$-bin or regime};
\draw[calcol!35,dash pattern=on 2.0pt off 1.8pt,line width=0.6pt]
      ([xshift=2.5mm]cal.west) -- ([xshift=-2.5mm]cal.east);
\node[fill=white,inner sep=1.6pt,font=\scriptsize\itshape,text=calcol!85]
      at (cal.center) {or};
\coordinate (jx) at ([xshift=-4.0mm]cal.west);
\coordinate (jm) at (jx |- cal.center);
\draw[lin,rounded corners=1.2pt] (d.east) -- (jx |- d) -- (jm);
\draw[lin,rounded corners=1.2pt] (b.east) -- (jx |- b) -- (jm);
\draw[arr] (jm) -- (cal.west);
\node[fin,anchor=west] (o1) at ([xshift=8mm,yshift= 6.7mm]cal.east)
  {\scriptsize $[\ell-\hat qa_\gamma,\ u+\hat qa_\gamma]$};
\node[fin,anchor=west] (o2) at ([xshift=8mm,yshift=-6.7mm]cal.east)
  {\scriptsize $[\ell-\hat q_h,\ u+\hat q_h]$};
\draw[calcol!40,densely dashed,line width=0.5pt] (o1.south) -- (o2.north);
\node[fill=white,inner sep=1.6pt,font=\scriptsize,text=calcol!85]
      at (o1 |- cal.center) {Calibrated interval};
\draw[arr] ([yshift= 6.7mm]cal.east) -- (o1.west);
\draw[arr] ([yshift=-6.7mm]cal.east) -- (o2.west);
\def\SS{1.80}
\def\PA{4.80}
\def\PB{8.60}
\def\PC{12.40}
\pgfmathsetmacro{\xBimg}{\PB+(6.85-7.6)*\SS}
\pgfmathsetmacro{\xBtab}{\PB+(7.60-7.6)*\SS}
\pgfmathsetmacro{\xBtxt}{\PB+(7.90-7.6)*\SS}
\pgfmathsetmacro{\xCtab}{\PC+(7.60-7.6)*\SS}
\pgfmathsetmacro{\xCtxt}{\PC+(7.90-7.6)*\SS}
\pgfmathsetmacro{\xCimg}{\PC-0.56}
\node[rowlab,anchor=west] at (-0.28,0.69) {Per-source $\hat y_k$};
\node[rowlab,anchor=west,text=calcol!90] at (-0.28,-0.30) {Scale or stratum};
\node[rowlab,anchor=west,text=basecol!90] at (-0.28,-0.95) {Base interval $[\ell,u]$};
\node[rowlab,anchor=west,text=black!52] at (-0.28,-1.65) {Global $\hat q$ (CQR, $\gamma{=}0$)};
\node[rowlab,anchor=west,text=calcol!90] at (-0.28,-2.35) {Modality-aware (ours)};
\node[ptitle] at (\PA,1.68) {Sources agree};
\node[psub]   at (\PA,-0.30) {$d(x){=}0.02,\ a_\gamma{=}1.00$};
\icImg{\PA-0.54}{1.15} \icTab{\PA}{1.15} \icTxt{\PA+0.54}{1.15}
\draw[lead] (\PA-0.54,0.95) -- ({\PA+(7.57-7.6)*\SS},0.77);
\draw[lead] (\PA,0.95)      -- ({\PA+(7.60-7.6)*\SS},0.77);
\draw[lead] (\PA+0.54,0.95) -- ({\PA+(7.63-7.6)*\SS},0.77);
\draw[black!22,line width=0.5pt] (\PA-1.60,0.69) -- (\PA+1.60,0.69);
\fill[imgcol] ({\PA+(7.57-7.6)*\SS},0.69) circle (1.6pt);
\fill[tabcol] ({\PA+(7.60-7.6)*\SS},0.69) circle (1.6pt);
\fill[txtcol] ({\PA+(7.63-7.6)*\SS},0.69) circle (1.6pt);
\node[font=\scriptsize,anchor=north] at (\PA,0.55)
   {\textcolor{imgcol}{7.57},\ \textcolor{tabcol}{7.60},\ \textcolor{txtcol}{7.63}};
\draw[basebar] ({\PA+(7.40-7.6)*\SS},-0.95) -- ({\PA+(7.80-7.6)*\SS},-0.95);
\draw[endt,basecol!85] ({\PA+(7.40-7.6)*\SS},-1.05) -- ({\PA+(7.40-7.6)*\SS},-0.85);
\draw[endt,basecol!85] ({\PA+(7.80-7.6)*\SS},-1.05) -- ({\PA+(7.80-7.6)*\SS},-0.85);
\node[flankL,text=basecol!90] at ({\PA+(7.40-7.6)*\SS-0.10},-0.95) {7.40};
\node[flankR,text=basecol!90] at ({\PA+(7.80-7.6)*\SS+0.10},-0.95) {7.80};
\draw[globbar] ({\PA+(7.15-7.6)*\SS},-1.65) -- ({\PA+(8.05-7.6)*\SS},-1.65);
\draw[endt,black!42] ({\PA+(7.15-7.6)*\SS},-1.74) -- ({\PA+(7.15-7.6)*\SS},-1.56);
\draw[endt,black!42] ({\PA+(8.05-7.6)*\SS},-1.74) -- ({\PA+(8.05-7.6)*\SS},-1.56);
\node[flankL,text=black!52] at ({\PA+(7.15-7.6)*\SS-0.10},-1.65) {7.15};
\node[flankR,text=black!52] at ({\PA+(8.05-7.6)*\SS+0.10},-1.65) {8.05};
\draw[calbar] ({\PA+(7.20-7.6)*\SS},-2.35) -- ({\PA+(8.00-7.6)*\SS},-2.35);
\draw[endt,calcol!85] ({\PA+(7.20-7.6)*\SS},-2.46) -- ({\PA+(7.20-7.6)*\SS},-2.24);
\draw[endt,calcol!85] ({\PA+(8.00-7.6)*\SS},-2.46) -- ({\PA+(8.00-7.6)*\SS},-2.24);
\node[flankL,text=calcol!90] at ({\PA+(7.20-7.6)*\SS-0.10},-2.35) {7.20};
\node[flankR,text=calcol!90] at ({\PA+(8.00-7.6)*\SS+0.10},-2.35) {8.00};
\node[ptitle] at (\PB,1.68) {Sources conflict};
\node[psub]   at (\PB,-0.30) {$d(x){=}0.44,\ a_\gamma{=}1.60$};
\icImg{\xBimg}{1.15} \icTab{\xBtab}{1.15} \icTxt{\xBtxt}{1.15}
\draw[lead] ({\PB+(6.85-7.6)*\SS},0.95) -- ({\PB+(6.85-7.6)*\SS},0.77);
\draw[lead] ({\PB+(7.60-7.6)*\SS},0.95) -- ({\PB+(7.60-7.6)*\SS},0.77);
\draw[lead] ({\PB+(7.90-7.6)*\SS},0.95) -- ({\PB+(7.90-7.6)*\SS},0.77);
\draw[black!22,line width=0.5pt] (\PB-1.60,0.69) -- (\PB+1.60,0.69);
\fill[imgcol] ({\PB+(6.85-7.6)*\SS},0.69) circle (1.6pt);
\fill[tabcol] ({\PB+(7.60-7.6)*\SS},0.69) circle (1.6pt);
\fill[txtcol] ({\PB+(7.90-7.6)*\SS},0.69) circle (1.6pt);
\node[font=\scriptsize,text=imgcol,anchor=north] at ({\PB+(6.85-7.6)*\SS},0.55) {6.85};
\node[font=\scriptsize,anchor=north] at ({\PB+(7.75-7.6)*\SS},0.55)
   {\textcolor{tabcol}{7.60},\ \textcolor{txtcol}{7.90}};
\draw[basebar] ({\PB+(7.40-7.6)*\SS},-0.95) -- ({\PB+(7.80-7.6)*\SS},-0.95);
\draw[endt,basecol!85] ({\PB+(7.40-7.6)*\SS},-1.05) -- ({\PB+(7.40-7.6)*\SS},-0.85);
\draw[endt,basecol!85] ({\PB+(7.80-7.6)*\SS},-1.05) -- ({\PB+(7.80-7.6)*\SS},-0.85);
\node[flankL,text=basecol!90] at ({\PB+(7.40-7.6)*\SS-0.10},-0.95) {7.40};
\node[flankR,text=basecol!90] at ({\PB+(7.80-7.6)*\SS+0.10},-0.95) {7.80};
\draw[globbar] ({\PB+(7.15-7.6)*\SS},-1.65) -- ({\PB+(8.05-7.6)*\SS},-1.65);
\draw[endt,black!42] ({\PB+(7.15-7.6)*\SS},-1.74) -- ({\PB+(7.15-7.6)*\SS},-1.56);
\draw[endt,black!42] ({\PB+(8.05-7.6)*\SS},-1.74) -- ({\PB+(8.05-7.6)*\SS},-1.56);
\node[flankL,text=black!52] at ({\PB+(7.15-7.6)*\SS-0.10},-1.65) {7.15};
\node[flankR,text=black!52] at ({\PB+(8.05-7.6)*\SS+0.10},-1.65) {8.05};
\draw[calbar] ({\PB+(7.08-7.6)*\SS},-2.35) -- ({\PB+(8.12-7.6)*\SS},-2.35);
\draw[endt,calcol!85] ({\PB+(7.08-7.6)*\SS},-2.46) -- ({\PB+(7.08-7.6)*\SS},-2.24);
\draw[endt,calcol!85] ({\PB+(8.12-7.6)*\SS},-2.46) -- ({\PB+(8.12-7.6)*\SS},-2.24);
\node[flankL,text=calcol!90] at ({\PB+(7.08-7.6)*\SS-0.10},-2.35) {7.08};
\node[flankR,text=calcol!90] at ({\PB+(8.12-7.6)*\SS+0.10},-2.35) {8.12};
\node[ptitle] at (\PC,1.68) {Modality missing};
\node[psub]   at (\PC,-0.30) {stratum $\hat q_h{=}0.30$};
\icImgOff{\xCimg}{1.15} \icTab{\xCtab}{1.15} \icTxt{\xCtxt}{1.15}
\draw[lead] ({\PC+(7.60-7.6)*\SS},0.95) -- ({\PC+(7.60-7.6)*\SS},0.77);
\draw[lead] ({\PC+(7.90-7.6)*\SS},0.95) -- ({\PC+(7.90-7.6)*\SS},0.77);
\draw[black!22,line width=0.5pt] (\PC-1.60,0.69) -- (\PC+1.60,0.69);
\fill[tabcol] ({\PC+(7.60-7.6)*\SS},0.69) circle (1.6pt);
\fill[txtcol] ({\PC+(7.90-7.6)*\SS},0.69) circle (1.6pt);
\node[font=\scriptsize,anchor=north] at ({\PC+(7.75-7.6)*\SS},0.55)
   {\textcolor{tabcol}{7.60},\ \textcolor{txtcol}{7.90}};
\draw[basebar] ({\PC+(7.40-7.6)*\SS},-0.95) -- ({\PC+(7.80-7.6)*\SS},-0.95);
\draw[endt,basecol!85] ({\PC+(7.40-7.6)*\SS},-1.05) -- ({\PC+(7.40-7.6)*\SS},-0.85);
\draw[endt,basecol!85] ({\PC+(7.80-7.6)*\SS},-1.05) -- ({\PC+(7.80-7.6)*\SS},-0.85);
\node[flankL,text=basecol!90] at ({\PC+(7.40-7.6)*\SS-0.10},-0.95) {7.40};
\node[flankR,text=basecol!90] at ({\PC+(7.80-7.6)*\SS+0.10},-0.95) {7.80};
\draw[globbar] ({\PC+(7.15-7.6)*\SS},-1.65) -- ({\PC+(8.05-7.6)*\SS},-1.65);
\draw[endt,black!42] ({\PC+(7.15-7.6)*\SS},-1.74) -- ({\PC+(7.15-7.6)*\SS},-1.56);
\draw[endt,black!42] ({\PC+(8.05-7.6)*\SS},-1.74) -- ({\PC+(8.05-7.6)*\SS},-1.56);
\node[flankL,text=black!52] at ({\PC+(7.15-7.6)*\SS-0.10},-1.65) {7.15};
\node[flankR,text=black!52] at ({\PC+(8.05-7.6)*\SS+0.10},-1.65) {8.05};
\draw[calbar] ({\PC+(7.10-7.6)*\SS},-2.35) -- ({\PC+(8.10-7.6)*\SS},-2.35);
\draw[endt,calcol!85] ({\PC+(7.10-7.6)*\SS},-2.46) -- ({\PC+(7.10-7.6)*\SS},-2.24);
\draw[endt,calcol!85] ({\PC+(8.10-7.6)*\SS},-2.46) -- ({\PC+(8.10-7.6)*\SS},-2.24);
\node[flankL,text=calcol!90] at ({\PC+(7.10-7.6)*\SS-0.10},-2.35) {7.10};
\node[flankR,text=calcol!90] at ({\PC+(8.10-7.6)*\SS+0.10},-2.35) {8.10};
\end{tikzpicture}}
\caption{\textbf{Modality-aware conformal calibration layer.} \emph{Top.} Per-source predictors produce $\hat y_k=g_k(x^{(k)})$; their pointwise standard deviation is the disagreement $d(x)$. A fixed base predictor supplies $[\ell(x),u(x)]$. Disagreement scaling uses $d(x)$ to adapt the conformal correction while retaining marginal split-conformal coverage, whereas Mondrian calibration uses fixed strata defined by disagreement or modality availability. \emph{Bottom.} In constructed $95\%$ examples with $[\ell,u]=[7.40,7.80]$ and $\gamma=8$, the raw-score quantile is $0.25$ and the scaled-score quantile is $0.20$. Scaling narrows the agreement interval by $11\%$ and widens the conflict interval by $16\%$; the missing-modality case instead uses its stratum-specific quantile $\hat q_h=0.30$.}
\label{fig:overview}
\end{figure}
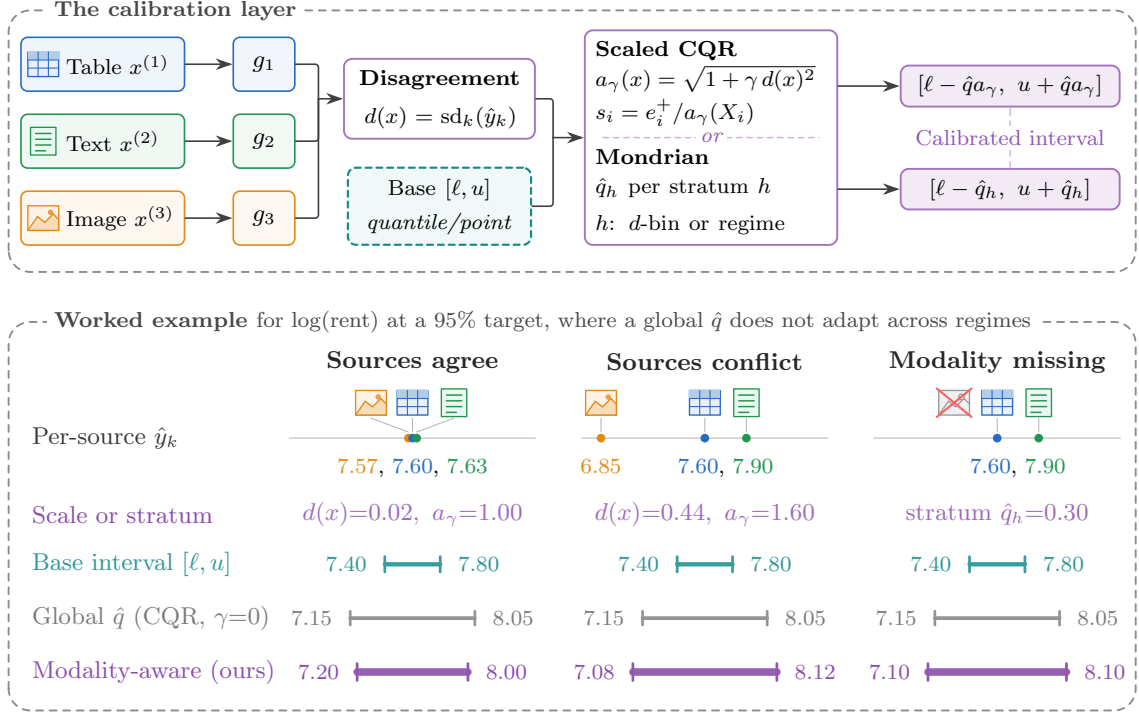

Reliable uncertainty estimates are essential when predictions inform consequential decisions. Distribution-free methods such as conformal prediction provide finite-sample coverage guarantees for arbitrary machine learning models \citep{vovk2005algorithmic,angelopoulos2023gentle}. Real prediction problems, however, often combine multiple input modalities, such as structured attributes, text descriptions, images, and sensor streams \citep{baltrusaitis2019survey}. Fusing these sources can improve point accuracy \citep{gao2020survey}, yet a single marginal coverage guarantee can mask systematic differences in interval reliability between regimes of the input \citep{vovk2005algorithmic,shafer2008tutorial,angelopoulos2023gentle}. In multi-modal systems, intervals should therefore remain reliable when modalities agree, when they conflict, and when a modality is missing at test time, a setting increasingly recognized as central in deployed models \citep{wu2026missingmodalitysurvey,ma2021smil,wang2023shaspec}.

We study whether disagreement among per-source predictors can serve as a useful calibration covariate. Agreement may indicate a stable regime, while conflict can make a single global quantile inefficient in easy cases and lead to undercoverage in harder subgroups despite valid marginal coverage \citep{vovk2012conditional,bostrom2020mondrian,barber2021limits}. Prior multi-modal work uses expected disagreement among separately trained unimodal classifiers as an ingredient in a lower bound on information synergy \citep{liang2024multimodal}. We instead use the dispersion of the per-source regression predictions at each example as a calibration covariate constructed before any conformal scores are computed. Recent work combines conformal regression with neural features from multi-modal inputs \citep{bose2024multimodalconformal}. In our setting, source disagreement and modality availability serve as calibration variables defined independently of the base predictor class. We use disagreement to continuously scale conformal scores, while both disagreement and modality availability can define fixed strata for Mondrian calibration \citep{vovk2005algorithmic,bostrom2020mondrian}. \autoref{fig:overview} illustrates the method. The paper makes four contributions.

\begingroup
\linespread{1.04}\selectfont
\begin{enumerate}
  \setlength{\itemsep}{2pt plus 1pt}
  \setlength{\parsep}{1pt}
  \setlength{\topsep}{3pt plus 1pt}
  \item We define a disagreement score at each example from per-source regression predictions for model-agnostic conformal calibration.
  \item We incorporate this score into a normalized split-conformal CQR method using a source-disagreement scale selected before calibration, and state the exact split conditions required for validity.
  \item We provide a Mondrian construction for coverage within pre-specified strata based on modality availability or disagreement.
  \item We evaluate the conformal calibration layer on the Swiss Real Estate Dataset (SRED), Mercari, Pawpularity, and IMDB-WIKI, including benchmarks across base predictors and stress tests with missing modalities that isolate the calibration wrapper from the fused predictor.
\end{enumerate}
\endgroup

\section{Background: Conformal Calibration}
\label{sec:cp-bg}

Given a predictor fixed before calibration and calibration observations exchangeable with the test observation, split conformal prediction constructs finite-sample valid prediction sets \citep{vovk2005algorithmic,shafer2008tutorial,lei2018distribution}. For ordered lower and upper estimates $(\ell(x),u(x))$ and calibration pairs $(X_i,Y_i)$ exchangeable with the test pair, conformalized quantile regression (CQR) uses calibration scores
\begin{equation}
  e_i=\max\{\ell(X_i)-Y_i,\;Y_i-u(X_i)\}
  \label{eq:cqr_score}
\end{equation}
and adjusts test intervals by the empirical split-conformal quantile \citep{romano2019cqr}. For a realized calibration observation we write the raw score as $e_i$. Throughout, we assume $\ell(x)\leq u(x)$. If fitted quantile estimates cross, we replace them before calibration by
\[
  \ell^\star(x)=\min\{\ell(x),u(x)\},\qquad
  u^\star(x)=\max\{\ell(x),u(x)\},
\]
and treat this deterministic rearrangement as part of the fixed prediction rule. We subsequently relabel $\ell^\star$ and $u^\star$ as $\ell$ and $u$. The score in \autoref{eq:cqr_score} may be negative when $Y_i$ already lies inside the base interval. This is standard CQR, and we refer to the score kept with its sign as the signed score. Clipping at zero is also valid but more conservative. Point predictors are covered as the degenerate case $\ell=u$. For a point predictor $\hat f$, setting $\ell(x)=u(x)=\hat f(x)$ gives $e_i=|Y_i-\hat f(X_i)|$, so split conformal prediction on absolute residuals is a special case of the same construction.

Mondrian conformal prediction computes separate calibration quantiles within pre-specified strata. The name comes from the Mondrian taxonomies of \citet{vovk2005algorithmic}, and \citet{vovk2012conditional} establishes category-wise validity for the inductive variant. The procedure can equally be read as stratified, or group-conditional, split conformal calibration \citep{bostrom2020mondrian}. Its guarantee is finite-stratum rather than arbitrary conditional coverage, which is impossible distribution-free without restrictions \citep{barber2021limits}. It gives a meaningful guarantee when the stratum rule, including any bin edges or clusters, is fixed before the calibration examples are scored. If the strata adapt to the data, their construction must use a separate tuning or reference split.

\section{Method}
\label{sec:method}

\paragraph{Problem setup.}
We observe random pairs $(X,Y)$ with $X\in\mathcal X$, $Y\in\mathbb{R}$, and $X=(X^{(1)},\ldots,X^{(K)})$ composed of $K$ source blocks. A source block is a fixed feature block with its own auxiliary predictor; it may represent an entire semantic modality (such as tabular, text, or image) or a distinct field within one. For an observed feature vector, let $x=(x^{(1)},\ldots,x^{(K)})$. Parenthesized superscripts index source blocks, whereas the subscript $i$ indexes observations. The data are divided into disjoint fitting, tuning, calibration, and test splits. The fitting split trains the base and source predictors, the tuning split absorbs every remaining data-driven choice (including early stopping, stacking or gating weights, and the choice of calibration rule), and the calibration split is reserved for conformal scores. Throughout, the calibration split denotes the held-out set that supplies the conformal scores, in the usual split conformal sense, and is always distinct from the tuning split. Baselines used only for point accuracy, which never enter the conformal layer, follow their own disclosed protocols (\appref{app:repro}). A fixed base predictor supplies ordered endpoints $(\ell(x),u(x))$ defining $[\ell(x),u(x)]$, with point predictors treated as the degenerate case $\ell=u$. For a prescribed miscoverage level $\alpha\in(0,1)$, the goal is to report intervals that are valid at coverage level $1-\alpha$ and remain efficient and reliable across regimes of source disagreement and modality availability.

The method is a conformal calibration layer around a fixed base predictor and per-source auxiliary predictors. It exposes source-wise disagreement, fixes either a scale rule or a finite stratum rule before calibration, and then applies split conformal prediction or CQR. This separation allows the base predictor to be a gradient-boosted tree, a quantile regressor, a neural network, or a Monte Carlo (MC) predictive sampler. The validity claim belongs entirely to the final conformal calibration step. The top of \autoref{fig:overview} sketches the pipeline.

Let $\mathcal D_{\mathrm{pre}}$ denote all data and algorithmic randomness used before calibration, including fitting data, tuning or reference data, pre-processing choices, fitted predictors, early stopping and model selection decisions, bin edges, definitions of missingness regimes, and the selected value of the disagreement-scale parameter $\gamma$ (\autoref{sec:weighted}). Conditional on $\mathcal D_{\mathrm{pre}}$, all functions used to score the calibration and test examples are fixed.

\subsection{Disagreement Between Sources}
\label{sec:disagreement}

On a fitting split, train per-source predictors $g_k:x^{(k)}\mapsto \hat y_k$, each predicting the target directly from one source block. Let $A(x)\subseteq\{1,\ldots,K\}$ be the set of sources available for example $x$, let $K_x=|A(x)|$, and define
\[
  \bar g_A(x)=\frac{1}{K_x}\sum_{k\in A(x)}g_k(x^{(k)}).
\]
The disagreement score is
\begin{equation}
  d(x)=
  \left\{
  \frac{1}{K_x}\sum_{k\in A(x)}
  \left(g_k(x^{(k)})-\bar g_A(x)\right)^2
  \right\}^{1/2}.
  \label{eq:disagreement}
\end{equation}
For $K_x=1$, this definition gives $d(x)=0$. If $K_x=0$, any branch whose score or stratum rule evaluates $d$, including disagreement scaling and scaled Mondrian calibration, is undefined and requires a pre-specified abstention or fallback interval. The corresponding guarantees assume non-empty availability for every calibration and test input. Signed Mondrian calibration on availability does not inherit this $d$-based restriction because neither its score nor its stratum rule evaluates $d$. Both the availability map $A(\cdot)$ and any fallback are part of $\mathcal D_{\mathrm{pre}}$ and must be applied identically to calibration and test examples. The score measures conflict among the per-source predictions $g_k(x^{(k)})$ of the same target, not the variance of one predictor. Comparing its magnitude across different availability patterns can be misleading, so all reported experiments with continuous scaling use one fixed, fully observed source set per dataset. The stress tests with missing modalities instead use the regime label $H_{\mathrm{reg}}$ of \autoref{sec:mondrian} after applying the same pre-specified mask to calibration and test examples. No prediction is invented for an absent source.

\subsection{Disagreement-scaled CQR}
\label{sec:weighted}

The disagreement score enters calibration through the positive scale
\begin{equation}
  a_\gamma(x)=\sqrt{1+\gamma d(x)^2},\qquad \gamma\geq 0,
  \label{eq:scale}
\end{equation}
where $\gamma$ is selected on the tuning split and fixed before calibration. Setting $\gamma=0$ recovers marginal calibration with the clipped score defined below. Dividing $d$ by a fixed positive reference scale $c_d$ can be absorbed into $\gamma$, so $\gamma$ is dimensionless for standardized $d$ and otherwise has units $d^{-2}$. More generally, any finite, measurable, strictly positive scale function fixed before calibration gives the same split-conformal guarantee. We use \autoref{eq:scale} because it is the standard deviation implied by an additive-noise variance that grows linearly in $d(x)^2$, has a non-zero floor at $d(x)=0$, and is the normalized representative of the two-parameter family $\sqrt{\lambda_0+\lambda_1d(x)^2}$ with constants $\lambda_0>0$ and $\lambda_1\geq0$. A common rescaling is absorbed by the conformal quantile, leaving only $\gamma=\lambda_1/\lambda_0$.

For calibration point $i$, let $e_i$ be the CQR score of \autoref{eq:cqr_score}. To prevent greater disagreement from inducing a stronger contraction, pre-specify the clipped score
\begin{equation}
  e_i^+=\max\{e_i,0\},
  \qquad
  s_i=\frac{e_i^+}{a_\gamma(X_i)}.
  \label{eq:scaled_score}
\end{equation}
This normalized non-conformity score uses explicit source disagreement rather than a generic covariate distance or black-box residual model \citep{papadopoulos2008normalized,guan2023localized,colombo2023localcp}. Viewed as a random variable of $(X_i,Y_i)$, this clipped, scaled score is written $S_i^{(\gamma)}$. For calibration size $n$, set $m=\lceil(n+1)(1-\alpha)\rceil$. If $m\leq n$, let $\hat q$ be the $m$th smallest scaled score. Otherwise set $\hat q=+\infty$. Since $\hat q\geq0$ and the base endpoints are ordered as in \autoref{sec:cp-bg}, the set defined by the score threshold and the equivalent interval, which is always non-empty, are
\begin{align}
  C_\gamma(x)
  &=\left\{y:
    \frac{\max\{\ell(x)-y,\;y-u(x),0\}}{a_\gamma(x)}
    \leq \hat q\right\},
  \label{eq:weighted_set}\\
  C_\gamma(x)
  &=\left[\ell(x)-\hat q\,a_\gamma(x),\;
          u(x)+\hat q\,a_\gamma(x)\right].
  \label{eq:weighted_interval}
\end{align}
Clipping is deliberate. Signed CQR can contract an overconservative base interval, but after disagreement normalization a negative quantile would make larger $d(x)$ produce a larger contraction, reversing the intended uncertainty ordering. The non-negative score prevents this inversion without changing the split-conformal proof. Compared with signed CQR under the same fixed scale rule, it is more conservative only when the signed quantile is negative. Both the cross-dataset benchmark and the SRED per-predictor diagnostic use this score during tuning and calibration. The selected scales and final quantiles are recorded for every run (\autoref{sec:setup}, \appref{app:repro}).

\begin{proposition}\label{prop:weighted}
Assume the base interval $(\ell,u)$, the source predictors defining $d$, the pre-processing needed to evaluate them, and $\gamma$ are fixed before computing calibration scores or their quantile, and that $A(x)\neq\emptyset$ for every calibration and test example. If the calibration examples and test example are exchangeable conditional on $\mathcal D_{\mathrm{pre}}$, then
\[
  \Pr\{Y_{n+1}\in C_\gamma(X_{n+1})\mid \mathcal D_{\mathrm{pre}}\}\geq 1-\alpha .
\]
\end{proposition}

\noindent The probability is over the calibration examples and the test example conditional on $\mathcal D_{\mathrm{pre}}$. It is not conditional on the realized calibration scores or on a fixed feature value $X=x$.

\begin{proof}
Conditional on $\mathcal D_{\mathrm{pre}}$, the functions $\ell,u,d,a_\gamma$ are fixed, so the clipped, scaled scores
\[
S_i^{(\gamma)}
=
\frac{\max\{\ell(X_i)-Y_i,\;Y_i-u(X_i),0\}}{a_\gamma(X_i)},
\qquad i=1,\ldots,n+1,
\]
are exchangeable. The event $Y_{n+1}\in C_\gamma(X_{n+1})$ is exactly $S_{n+1}^{(\gamma)}\leq\hat q$, so the quantile lemma for exchangeable scores gives the claim \citep[Lemma~1]{tibshirani2019covariate}; Theorem~2.2 of \citet{lei2018distribution} gives the standard split-conformal regression specialization. If $m>n$, it is immediate because $\hat q=+\infty$.
\end{proof}

In the experiments, $\gamma$ is selected from a finite grid on the tuning split. \appref{app:repro} lists the grids and maps the tuned two-parameter form to \autoref{eq:scale}. Split separation is essential because tuning $\gamma$ on the labels used for $\hat q$ would break the exact proof.

\subsection{Disagreement and Regime-Mondrian Calibration}
\label{sec:mondrian}

Let $H:\mathcal X\to\mathcal H$ be a fixed map to a completely specified finite label set $\mathcal H$. We use two versions:
\[
  \begin{aligned}
  H_{\mathrm{dis}}(x)&=\text{disagreement stratum determined by fixed cut points},\\
  H_{\mathrm{reg}}(x)&=\text{modality-availability or masking regime}.
  \end{aligned}
\]
For disagreement strata, let $-\infty=b_0\leq b_1\leq\cdots\leq b_J=+\infty$ be cut points fixed before calibration and set $H_{\mathrm{dis}}(x)=j$, $j\in\{1,\ldots,J\}$, when $b_{j-1}\leq d(x)<b_j$. An observation on an edge therefore enters the upper stratum. Repeated cut points create an empty stratum, which remains in the fixed codomain rather than being silently merged. The codomain of $H_{\mathrm{reg}}$ is likewise a pre-specified list of modality-availability patterns. Empty or undersized strata receive $\hat q_h=+\infty$ under the convention below. The formal construction does not pool them with another stratum.
In the signed version, compute $\hat q_h$ as the split-conformal order statistic of the CQR scores $e_i$ restricted to calibration points satisfying $H(X_i)=h$. Define the Mondrian conformal set by the score threshold
\[
  C_h(x)
  =
  \left\{
  y:
  \max\{\ell(x)-y,\;y-u(x)\}\leq \hat q_h
  \right\}.
\]
When non-empty, this set has endpoints $[\ell(x)-\hat q_h,\;u(x)+\hat q_h]$.
In the scaled version, compute $\hat q_h$ from the non-negative scaled scores $s_i$ restricted to $H(X_i)=h$, with $\gamma$ selected on the tuning split and fixed exactly as in \autoref{sec:weighted}, and define
\[
C_h(x)=
\left\{y:
\frac{\max\{\ell(x)-y,\;y-u(x),0\}}{a_\gamma(x)}
\leq \hat q_h
\right\}.
\]
Its endpoint representation is $[\ell(x)-\hat q_h a_\gamma(x),\;u(x)+\hat q_h a_\gamma(x)]$.

\begin{proposition}\label{prop:mondrian}
For any $h\in\mathcal H$ satisfying, almost surely on the realizations under consideration,
\[
\Pr\{H(X_{n+1})=h\mid\mathcal D_{\mathrm{pre}}\}>0,
\]
let $n_h$ be the number of calibration examples with $H(X_i)=h$ and let $m_h=\lceil(n_h+1)(1-\alpha)\rceil$. If $m_h\leq n_h$, set $\hat q_h$ to the $m_h$th order statistic of the chosen signed or scaled scores in that stratum, with $C_h$ the corresponding conformal set above. Otherwise set $\hat q_h=+\infty$. Suppose $H$ is fixed using only data in $\mathcal D_{\mathrm{pre}}$ and the calibration examples and test example are exchangeable conditional on $\mathcal D_{\mathrm{pre}}$. Whenever evaluating $H$ or the chosen score requires $d$, also assume $A(X_i)\neq\emptyset$ for $i=1,\ldots,n+1$. Then
\[
  \Pr\{Y_{n+1}\in C_h(X_{n+1})
  \mid H(X_{n+1})=h,\mathcal D_{\mathrm{pre}}\}
  \geq 1-\alpha .
\]
\end{proposition}

\begin{proof}
Condition on $\mathcal D_{\mathrm{pre}}$ and on the realized membership vector $(H(X_1),\ldots,H(X_{n+1}))$ with $H(X_{n+1})=h$. This fixes the in-stratum index set $I_h=\{i\leq n:H(X_i)=h\}$ and $n_h=|I_h|$. Since $H$ is a fixed function given $\mathcal D_{\mathrm{pre}}$, this conditioning event is invariant under every permutation of the examples indexed by $I_h\cup\{n+1\}$. Each such example enters the event only through the symmetric requirement $H(X)=h$. Hence, by the assumed exchangeability of the $n+1$ examples given $\mathcal D_{\mathrm{pre}}$, the examples indexed by $I_h\cup\{n+1\}$ remain exchangeable under this conditioning, and so do their scores. Applying the same exchangeability argument as in Proposition~\ref{prop:weighted} to the $n_h$ calibration scores in the stratum and the test score gives coverage at least $m_h/(n_h+1)\geq 1-\alpha$. The claim is immediate when $m_h>n_h$ because then $\hat q_h=+\infty$. Averaging over the membership vectors consistent with $H(X_{n+1})=h$ gives the statement \citep[cf.][]{vovk2012conditional}.
\end{proof}

Three remarks delimit the guarantee. First, a finite $95\%$ split-conformal quantile requires at least $19$ calibration scores overall and in each protected stratum. Pooling small strata forfeits the stated guarantee conditional on the stratum. The simulation in \appref{app:smallsample} studies calibration budgets near this floor. Second, Proposition~\ref{prop:mondrian} controls one future interval at a time. For a pre-specified batch of $B$ intervals, calibration at miscoverage level $\alpha/B$ and the union bound control the probability of any miscoverage by $\alpha$ without assuming independent errors. Taking $\alpha/|\mathcal H|$ is the special case with one interval per stratum. Exchangeability is still required, and Bonferroni correction cannot repair a stratum rule that depends on the data or create validity conditional on the covariates. Its smaller per-interval level also increases the required calibration size. Third, corrections for simultaneous coverage are unrelated to testing multiple metrics across correlated seeds or base predictors.

\raggedbottom

Algorithm~\ref{algo:macc} summarizes the full procedure with every pre-calibration choice fixed.

\begingroup
\setlength{\intextsep}{6pt}
\begin{algorithm}[H]
\small
\caption{Modality-aware conformal calibration}
\label{algo:macc}
\begin{algorithmic}[1]
\REQUIRE Disjoint fit, tune, and calibration splits $\mathcal D_{\mathrm{fit}},\mathcal D_{\mathrm{tune}},\mathcal D_{\mathrm{cal}}$, with $n$ calibration observations; test input $x$; miscoverage $\alpha$; mode $\in\{\text{scaled},\text{Mondrian}\}$; for Mondrian mode, score variant $\in\{\text{signed},\text{scaled}\}$; whenever disagreement is required, assume $A(X_i)\neq\emptyset$ for every calibration input
\ENSURE Calibrated prediction set $C(x)$
\STATE Fit the base interval rule and every required source predictor $g_k$ on $\mathcal D_{\mathrm{fit}}$, with any pre-specified use of $\mathcal D_{\mathrm{tune}}$ such as early-stopping monitoring or the fitting of stacking or gating weights; rearrange the base endpoints so that $\ell(v)\leq u(v)$
\STATE Use $\mathcal D_{\mathrm{tune}}$ for all remaining data-driven choices, including $\gamma$ when scaling is used and the complete Mondrian rule $H:\mathcal X\to\mathcal H$; freeze all choices before calibration
\STATE When disagreement is required, form $A(X_i)$ for each calibration input and $A(x)$ for the test input
\IF{disagreement is required and $A(x)=\emptyset$}
    \STATE \textbf{return} the pre-specified fallback interval $C_{\mathrm{fb}}(x)$; this branch lies outside Propositions~\ref{prop:weighted} and~\ref{prop:mondrian}
\ENDIF
\STATE Compute $d(X_i)$ and $d(x)$ from \autoref{eq:disagreement} when required
\STATE For every $(X_i,Y_i)\in\mathcal D_{\mathrm{cal}}$, compute $e_i=\max\{\ell(X_i)-Y_i,\;Y_i-u(X_i)\}$
\IF{mode $=$ scaled}
    \STATE Set $a_\gamma(v)=\sqrt{1+\gamma d(v)^2}$ and $s_i=\max\{e_i,0\}/a_\gamma(X_i)$
    \STATE Set $m=\lceil(n+1)(1-\alpha)\rceil$ and $\hat q$ to the $m$th smallest $s_i$, or to $+\infty$ if $m>n$
    \STATE \textbf{return} $C(x)=[\ell(x)-\hat q\,a_\gamma(x),\;u(x)+\hat q\,a_\gamma(x)]$
\ELSE
    \STATE For the signed variant, set $t_i=e_i$ and $R(v,y)=\max\{\ell(v)-y,\;y-u(v)\}$
    \STATE For the scaled variant, set $a_\gamma(v)=\sqrt{1+\gamma d(v)^2}$, $t_i=s_i=\max\{e_i,0\}/a_\gamma(X_i)$, and $R(v,y)=\max\{\ell(v)-y,\;y-u(v),0\}/a_\gamma(v)$
    \FOR{each pre-specified stratum $h\in\mathcal H$}
        \STATE Set $I_h=\{i:H(X_i)=h\}$, $n_h=|I_h|$, and $m_h=\lceil(n_h+1)(1-\alpha)\rceil$
        \STATE Let $\hat q_h$ be the $m_h$th smallest $t_i$ for $i\in I_h$, or $+\infty$ if $m_h>n_h$
    \ENDFOR
    \STATE Set $h=H(x)$; \textbf{return} $C(x)=\{y:R(x,y)\leq\hat q_h\}$
\ENDIF
\end{algorithmic}
\end{algorithm}
\endgroup
\FloatBarrier

\section{Related Work}
\label{sec:related}

Inductive and split conformal prediction provide finite-sample marginal coverage for a fixed rule \citep{vovk2005algorithmic,papadopoulos2002inductive,shafer2008tutorial,lei2018distribution,angelopoulos2023gentle}, while conformalized quantile regression calibrates fitted quantile endpoints \citep{romano2019cqr}. Coverage conditional on the realized training sample has also been analyzed \citep{bian2023training}, and impossibility results delimit unrestricted distribution-free conditional validity \citep{barber2021limits}. Mondrian and group-conditional methods provide finite-stratum guarantees when the grouping rule is fixed before calibration \citep{vovk2012conditional,bostrom2020mondrian}. \citet{bellotti2021approximation} instead studies an empirical approximation to object-conditional validity.

Our layer applies these ideas to source disagreement and modality availability. Its continuous branch relates to normalized and localized conformal methods \citep{papadopoulos2008normalized,guan2023localized,colombo2023localcp}, but uses the dispersion of source predictions rather than a covariate distance, residual model, or estimated likelihood ratio \citep{tibshirani2019covariate}. Prior methods use estimates of instance difficulty in several ways. \citet{bostrom2020mondrian} form Mondrian categories by grouping calibration instances according to estimated difficulty, with approximately the same number in each category. \citet{johansson2024reject} evaluate the variance of individual random-forest tree predictions as a difficulty estimate in a conformal framework with a reject option. Our signal instead measures variation across target predictions from models fitted to different input blocks. A block may represent an entire input modality, such as tabular, text, or image features, or a distinct field within one modality, such as a title or description. The corresponding block-specific models may be components of the calibrated predictor or auxiliary models fitted to construct the disagreement score.

Multi-modal learning fuses tabular, text, image, and other input channels \citep{baltrusaitis2019survey,gao2020survey}. Aggregate disagreement between separately trained unimodal classifiers has been used to quantify multi-modal interaction in classification \citep{liang2024multimodal}. Our score is instead a per-example regression calibration covariate. Missing-modality methods address sources absent or corrupted at inference time \citep{wu2026missingmodalitysurvey}. Representative approaches use generative latent variables, stochastic fusion, training under severe missingness, or representations that separate shared and specific features \citep{wu2018mvae,choi2019embracenet,ma2021smil,wang2023shaspec}. Closely related multi-modal work by \citet{bose2024multimodalconformal} constructs conformal intervals from internal neural features of models with image or text inputs. We instead expose source-wise predictions, allowing the same disagreement score to work with tree models or MC samplers and making availability regimes explicit.

For conformal regression with missing covariates, \citet{zaffran2023missing} show that coverage can fail conditionally on the missingness pattern and propose missing data augmentation. Recent conformal methods that condition on the missingness mask address general missingness mechanisms through distributional imputation and weighted or acceptance-rejection corrections \citep{fan2026maskconditional}. Our construction instead targets a small, pre-specified set of masks that remove whole modalities and computes direct quantiles within each stratum. The reported stress tests with fixed masks use its mask-matched special case with a constant stratum label.

Conformal predictive distributions and high-dimensional conformal regions extend conformal inference beyond the scalar prediction intervals considered in this work \citep{vovk2017cpd,vovk2020efficientcps,bostrom2021mondrianpd,izbicki2022hpd}. Deep ensembles and normalizing flows are alternative models of predictive uncertainty \citep{lakshminarayanan2017deep,papamakarios2021flows}, while pair-copula constructions are flexible components for modeling multivariate dependence \citep{aas2009paircopula}. CLEAR jointly tunes a global scale and the relative weight of separately estimated aleatoric and epistemic components \citep{azizi2026clear}. The two-parameter scale family of \appref{app:repro} plays the analogous role in our continuous branch, with observable disagreement among per-source predictors in place of an estimated epistemic component. The Supervised Expectation-Maximization Framework (SEMF) offers another model-agnostic approach to prediction intervals through latent-variable modeling \citep{azizi2025semf}. Its representations for separate input groups make it relevant to multi-modal inputs and missing data, although the original study presents those settings as future work rather than evaluating them. In the SRED diagnostics, we treat ensembles of SEMF decoder outputs as fixed base predictions and calibrate them conformally (\appref{app:semf}). Turning these models into distribution-free intervals still requires an appropriate conformal score.

\section{Experiments}
\label{sec:exp}

\subsection{Experimental Setup}
\label{sec:setup}

The experiments use four multi-modal regression datasets. The main case study is SRED, with 11,105 Swiss apartment rental listings \citep{azizi2022sred}. Its target is the natural logarithm of rent, $\log(\mathrm{rent})$, and its representation combines tabular variables, text embeddings, and image embeddings. The benchmark further includes Mercari \citep{mercari2018kaggle}, Pawpularity \citep{pawpularity2021kaggle}, and IMDB-WIKI \citep{rothe2018deep}. \autoref{tab:datasets} in \appref{app:repro} gives the target variables and the disjoint fit, tune, calibration, and test split sizes used in the cross-dataset benchmark, and overlap checks between the test and training splits are reported in the same appendix.

The experiments operate on fixed representations rather than training on raw pixels or raw text. The cross-dataset runs use frozen pretrained embeddings directly, with multiple fields mean-pooled within a modality. A separate SRED diagnostic additionally evaluates the wrapper per base predictor on a precomputed 36-dimensional representation (\appref{app:repro}, \appref{app:additional}).

\paragraph{Metrics and calibration protocol.}
All experiments on the four datasets use five seeds (listed in \appref{app:repro}) and target nominal coverage $1-\alpha=0.95$. Values written as $a\pm b$ and figure error bars report the mean and one sample standard deviation as descriptive variation across runs, not as standard errors, confidence intervals, or inferential uncertainty statements. The cross-dataset benchmarks use disjoint fit, tune, calibration, and test splits. For the disagreement-Mondrian experiments, we use three coarse strata defined on the tuning split by the empirical $1/3$ and $2/3$ disagreement quantiles. This balances resolution and calibration size, but is not required by Proposition~\ref{prop:mondrian}. The edges are fixed before calibration. The continuous disagreement scale is likewise selected from a finite grid on the tuning split. \appref{app:repro} gives the precise convention for empirical quantiles, grids, and objective. All final $\hat q$ values for the disagreement-scaled runs are non-negative, and endpoint metrics are computed only after verifying $L_j\leq U_j$.

We report empirical coverage (PICP), mean prediction interval width (MPIW), normalized calibrated interval width (NCIW) \citep{azizi2026clear}, and interval continuous ranked probability score (CRPS), following the calibration--sharpness principle of \citet{gneiting2007calibration}. Interval CRPS is the CRPS of the uniform distribution over the reported interval. \appref{app:metrics} gives the closed-form definitions.

The auxiliary SRED per-predictor diagnostic (\appref{app:additional}) compares three base predictors, an XGBoost point model \citep{chen2016xgboost}, SEMF-derived MC ensembles of decoder outputs, and an XGBoost quantile model. \appref{app:repro} and \appref{app:semf} give the decoder variants, splits, and details of the precomputed representation. The principal calibration variants are marginal split conformal/CQR, disagreement-Mondrian split conformal/CQR, and disagreement-scaled split conformal/CQR, according to whether the base predictor supplies a point or quantile interval.

A monotone piecewise-constant disagreement scale selected on the tuning split is also evaluated as a secondary binned variant. Because it improves on the continuous scale in only 3 of the 60 paired CRPS comparisons, only the continuous scale is reported in the table, and the binned variant's paired counts appear in \appref{app:repro}. In the cross-dataset benchmark, a separate sensitivity comparator fits a shallow XGBoost model to the fused covariates of the tuning split to predict the log of the non-negative residual or CQR score. Its exponentiated, median-normalized prediction defines a positive normalizer that is frozen before calibration. The marginal comparator uses absolute residuals for point predictors and signed CQR scores for endpoints from quantile models. Every final signed marginal quantile in the paired cross-dataset and SRED diagnostic runs is non-negative, so these endpoints coincide with the clipped construction at $\gamma=0$ and whenever the selected two-parameter scale of \appref{app:repro} has $a_1=0$, equivalently $\gamma=0$ in \autoref{eq:scale}.

The form of \autoref{eq:disagreement} with varying availability beyond a fixed source set and the scaled Mondrian variant of \autoref{sec:mondrian} are stated as definitions only and are not exercised by these experiments. Formal guarantees require model fitting, tuning, and calibration to be separated. The SRED per-predictor rows are therefore empirical diagnostics because their calibration split was not kept untouched throughout SEMF fitting and tuning.

\subsection{Disagreement-scaled Calibration Across Datasets and Predictors}
\label{sec:main-results}

The headline cross-dataset benchmark applies the wrapper to three base predictors per dataset. The first is an XGBoost point predictor trained on the fitting split with early stopping monitored on the tuning split. The second is an XGBoost quantile model trained on the fitting split, whose fitted endpoints are calibrated with CQR scores. The third is a source-wise quantile ensemble that combines per-source quantile predictors trained on the fitting split, using fixed convex weights inversely proportional to the RMSE of their median predictions on the tuning split. A fourth base predictor, a ridge stacker of per-source point predictions fitted on the tuning split, provides a robustness check (\appref{app:repro}). All such choices are fixed before calibration.

\begin{table}[!htbp]
\centering
\small
\setlength{\tabcolsep}{3pt}
\caption{Marginal and disagreement-scaled calibration across four multi-modal datasets. Each dataset aggregates XGBoost point, XGBoost quantile, and source-wise quantile predictors over five seeds (15 paired predictor--seed runs); metric entries are the mean $\pm$ one sample standard deviation. Metric values are not highlighted by rank: PICP should be near $0.95$, whereas lower is better for MPIW, NCIW, and CRPS. The columns on the right report W/T/L against marginal calibration for the same predictor--seed run. W is a strict decrease, T an exact tie, and L an increase at full precision; every triplet sums to 15. The bold W/T counts count the runs that match or improve, as summarized in the text, while ties remain distinct from wins. MPIW and CRPS have dataset-specific target units, and NCIW is dimensionless.}
\label{tab:cross-predagn}
\begin{tabular*}{\linewidth}{@{\extracolsep{\fill}}llcccc}
\toprule
\addlinespace[3pt]
\multicolumn{6}{@{}l}{\normalsize\itshape Panel A: Coverage and raw interval width} \\[1pt]
Dataset & Calibration & PICP & MPIW & \multicolumn{2}{c}{MPIW} \\
 & & & & \multicolumn{2}{c}{W / T / L} \\
\midrule
SRED & Marginal & $0.943\pm0.005$ & $0.714\pm0.136$ & \multicolumn{2}{c}{\textit{Reference}} \\
 & Disagreement-scaled & $0.942\pm0.006$ & $0.689\pm0.128$ & \multicolumn{2}{c}{$\mathbf{15/0}/0$} \\
\cmidrule{1-6}
Mercari & Marginal & $0.950\pm0.002$ & $2.439\pm0.083$ & \multicolumn{2}{c}{\textit{Reference}} \\
 & Disagreement-scaled & $0.950\pm0.002$ & $2.439\pm0.084$ & \multicolumn{2}{c}{$\mathbf{6/3}/6$} \\
\cmidrule{1-6}
Pawpularity & Marginal & $0.953\pm0.005$ & $84.426\pm4.608$ & \multicolumn{2}{c}{\textit{Reference}} \\
 & Disagreement-scaled & $0.952\pm0.006$ & $83.020\pm5.599$ & \multicolumn{2}{c}{$\mathbf{9/5}/1$} \\
\cmidrule{1-6}
IMDB-WIKI & Marginal & $0.952\pm0.003$ & $40.616\pm4.799$ & \multicolumn{2}{c}{\textit{Reference}} \\
 & Disagreement-scaled & $0.951\pm0.003$ & $40.464\pm4.650$ & \multicolumn{2}{c}{$\mathbf{9/5}/1$} \\
\midrule
\addlinespace[3pt]
\multicolumn{6}{@{}l}{\normalsize\itshape Panel B: Normalized interval width and CRPS} \\[1pt]
Dataset & Calibration & NCIW & CRPS & NCIW & CRPS \\
 & & & & W / T / L & W / T / L \\
\midrule
SRED & Marginal & $0.349\pm0.072$ & $0.103\pm0.020$ & \multicolumn{2}{c}{\textit{Reference}} \\
 & Disagreement-scaled & $0.341\pm0.069$ & $0.102\pm0.020$ & $\mathbf{12/0}/3$ & $\mathbf{15/0}/0$ \\
\cmidrule{1-6}
Mercari & Marginal & $0.326\pm0.011$ & $0.373\pm0.025$ & \multicolumn{2}{c}{\textit{Reference}} \\
 & Disagreement-scaled & $0.326\pm0.011$ & $0.373\pm0.025$ & $\mathbf{9/3}/3$ & $\mathbf{12/3}/0$ \\
\cmidrule{1-6}
Pawpularity & Marginal & $0.858\pm0.048$ & $12.791\pm2.123$ & \multicolumn{2}{c}{\textit{Reference}} \\
 & Disagreement-scaled & $0.847\pm0.054$ & $12.717\pm2.188$ & $\mathbf{8/5}/2$ & $\mathbf{9/5}/1$ \\
\cmidrule{1-6}
IMDB-WIKI & Marginal & $0.447\pm0.052$ & $6.489\pm1.294$ & \multicolumn{2}{c}{\textit{Reference}} \\
 & Disagreement-scaled & $0.446\pm0.051$ & $6.478\pm1.286$ & $\mathbf{7/5}/3$ & $\mathbf{10/5}/0$ \\
\bottomrule
\end{tabular*}

\end{table}

\autoref{tab:cross-predagn} isolates the conformal wrapper from the base predictors by pairing calibrations after each predictor is fixed. Across four datasets, three base predictors per dataset, and five seeds, disagreement scaling yields 39 strict width improvements, 13 exact ties, and 8 losses; for CRPS it yields 46 strict improvements, 13 exact ties, and 1 loss. It therefore matches or improves marginal calibration in 52 of 60 width comparisons and 59 of 60 CRPS comparisons. NCIW has 36 strict improvements, 13 ties, and 11 losses, or 49 of 60 non-inferior comparisons. Overall mean PICP changes only from $0.94943$ to $0.94899$. Adding the ridge stacker gives the same qualitative robustness pattern: 49/19/12 strict improvements/ties/losses for width and 56/19/5 for CRPS, hence 68 of 80 and 75 of 80 non-inferior comparisons. These are descriptive paired summaries, not a pooled significance test: predictors within each dataset--seed cell share splits and source predictions, and the five seeds are not independent dataset draws.

Diagnostics by disagreement bin explain the aggregate mechanism (\autoref{fig:mechanism}). Marginal coverage in the high bin averages $0.923$, compared with $0.937$ under disagreement scaling; the latter retains $0.954$ and $0.955$ coverage in the low and middle bins. Disagreement-Mondrian calibration reaches $0.948$ in the high bin when stratum-wise coverage is the priority, at the price of wider high-disagreement intervals. \autoref{fig:examples} in \appref{app:worked} shows both regimes in individual listings.

\begin{figure}[!htbp]
  \centering
  \includegraphics[width=0.98\linewidth]{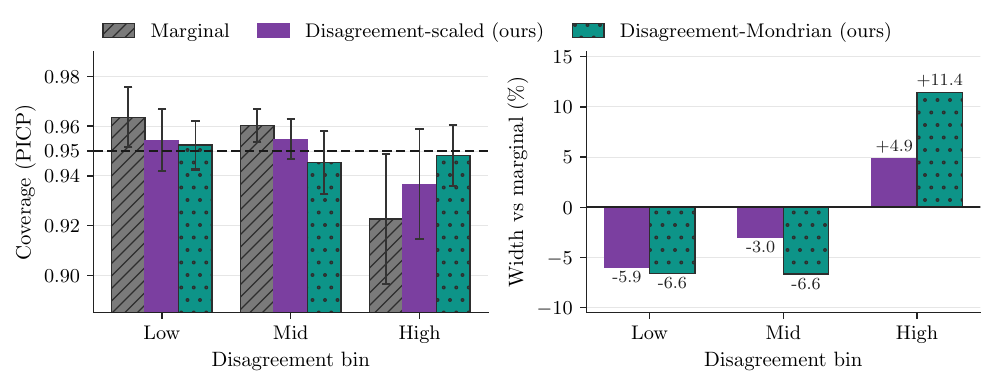}
  \caption{\textbf{The reallocation mechanism.} Diagnostics by disagreement bin across four datasets, three base predictors per dataset, and five seeds (60 matched predictor--dataset--seed runs). Whiskers show $\pm1$ sample standard deviation as descriptive variation across runs on a zoomed coverage axis, not standard errors or confidence intervals. \emph{Left.} Marginal calibration over-covers the low- and mid-disagreement bins and under-covers the high bin; disagreement scaling and disagreement-Mondrian calibration move coverage toward $0.95$. \emph{Right.} Both narrow easy bins and widen the high-disagreement bin, with greater widening under the stratum guarantee.}
  \label{fig:mechanism}
\end{figure}

As an additional sensitivity check, \autoref{tab:cross-predagn-sensitivity} (\appref{app:xgb-scale}) reports normalized split conformal/CQR using an XGBoost-learned scale based on fused features. For width, it yields 9 strict improvements, 4 exact ties, and 47 losses; for CRPS, the corresponding counts are 10, 4, and 46. This particular learned normalizer is unstable, especially for CQR on Pawpularity, and does not characterize learned normalization methods more generally. The SRED per-predictor diagnostic in \appref{app:additional} shows the same qualitative pattern as the headline disagreement analysis. Disagreement-scaled split conformal/CQR attains the lowest mean NCIW and mean raw width for all three base predictors in that diagnostic, while disagreement-Mondrian provides the explicitly stratified validity mechanism. The key claim has two parts. Disagreement scaling is the empirical efficiency variant, while Mondrian calibration is the finite-stratum conditional repair when strata are fixed before calibration.

\subsection{Missing-modality Regimes}
\label{sec:regimes}

The strongest multi-modal failure mode appears when a source is absent at test time. The experiment in \autoref{tab:missing-cross} trains the full-modality XGBoost point predictor of \autoref{sec:main-results} and compares three calibration choices. These are a single full-regime residual quantile, one pooled quantile over all represented masked calibration examples, and mask-matched quantiles computed after applying the same missing-modality mask to the calibration split. For each mask, the corresponding modality block of frozen embeddings is set to zero in calibration and test examples. The tabular block is never masked, and no missingness indicator is added. The pooled comparator concatenates, with equal weight, the absolute-residual scores recomputed on the same calibration examples under each distinct non-full mask, excluding the unmasked regime, so every mask contributes one copy of the $n$ calibration scores and the pooled size is $N_{\mathrm{pool}}=3n$ for SRED and IMDB-WIKI and $N_{\mathrm{pool}}=n$ for Mercari and Pawpularity. Its threshold is the $\lceil(N_{\mathrm{pool}}+1)(1-\alpha)\rceil$th smallest pooled score, the same finite-sample order-statistic convention used for every split-conformal quantile in the paper.

Each table row is a separate experiment with a deterministic mask applied to every calibration and test example, not a partition of one sample with mixed missingness. Thus $H$ is constant within a row and $n_h=n$: mask-matched recalibration is the special case of Proposition~\ref{prop:mondrian} with a constant stratum label, equivalent to ordinary split conformal on the masked distribution rather than a nontrivial Mondrian partition. It uses no test labels. Deterministic masking preserves exchangeability when applied identically to calibration and test examples, but a changed missingness mechanism requires fresh representative calibration data.

\begin{table}[!htbp]
\centering
\small
\setlength{\tabcolsep}{2.4pt}
\caption{Cross-dataset stress test with missing modalities (five-seed means). Full uses the full-regime quantile, pooled heuristically combines all represented masked calibration scores, and mask-matched recalibration uses calibration examples transformed by the row's fixed mask. Gain is the coverage change in percentage points, and MPIW change is the relative width change from full to mask-matched calibration. Both average unrounded per-seed changes and can differ by $0.1$ from differences recomputed from the rounded columns. For the two $K=2$ datasets, Mercari and Pawpularity, pooled and mask-matched calibration coincide because only one non-full mask exists.}
\label{tab:missing-cross}
\begin{tabular*}{\linewidth}{@{\extracolsep{\fill}}llccccc}
\toprule
Dataset & Masked regime & Full PICP & Pooled PICP & Mask PICP & Gain (pp) & MPIW change \\
\midrule
SRED & no text & $0.907$ & $0.966$ & $0.942$ & $+3.4$ & $+21.1\%$ \\
SRED & no image & $0.866$ & $0.953$ & $0.942$ & $+7.6$ & $+40.7\%$ \\
SRED & tabular only & $0.790$ & $0.918$ & $0.946$ & $+15.6$ & $+71.2\%$ \\
Mercari & tabular only & $0.930$ & $0.950$ & $0.950$ & $+2.1$ & $+12.5\%$ \\
Pawpularity & tabular only & $0.938$ & $0.951$ & $0.951$ & $+1.4$ & $+13.5\%$ \\
IMDB-WIKI & no text & $0.950$ & $0.998$ & $0.950$ & $+0.0$ & $+0.6\%$ \\
IMDB-WIKI & no image & $0.759$ & $0.934$ & $0.953$ & $+19.4$ & $+123.2\%$ \\
IMDB-WIKI & tabular only & $0.758$ & $0.933$ & $0.953$ & $+19.5$ & $+125.0\%$ \\
\bottomrule
\end{tabular*}

\end{table}

\autoref{tab:missing-cross} shows the width cost of repairing coverage after a modality is removed. The full-regime quantile under-covers hard masks, while pooled masked calibration can over-widen easy masks and still miss hard ones. The largest repair is the IMDB-WIKI tabular-only regime, where mask-matched recalibration raises PICP from $0.758$ to $0.953$, a $19.5$ percentage-point gain, while increasing MPIW by $125.0\%$. For SRED and IMDB-WIKI, pooling stacks dependent masked copies of the same calibration examples, so we claim no finite-sample split-conformal guarantee for that comparator. For Mercari and Pawpularity, $K=2$ leaves one non-full mask, so pooled and mask-matched quantiles are identical ordinary split-conformal quantiles and inherit the same guarantee under exchangeability. On IMDB-WIKI no-text, pooling reaches $0.998$ coverage at twice the mask-matched width, whereas mask-matched recalibration leaves this nearly unaffected mask essentially unchanged. In this point-predictor experiment, both rules report a constant symmetric radius around the same point prediction within a fixed mask, and the mask-matched width increase reflects the larger residuals needed to recover nominal coverage. \autoref{fig:missing-dumbbell} displays the repair across all eight masks.

The same mask-matched repair appears for the SEMF-derived SRED predictor with the original decoder under simulated missingness (\autoref{fig:regime} and \autoref{tab:regime} in \appref{app:additional}). To avoid reusing the SEMF validation split, these rows divide the held-out test predictions into calibration and evaluation halves. For no-image examples, mask-matched calibration raises coverage from $0.923$ to $0.946$ while MPIW grows from $3.050$ to $3.393$. For tabular-only examples, coverage rises from $0.921$ to $0.947$ while MPIW grows from $3.045$ to $3.472$. The remaining deviations are consistent with finite held-out evaluation samples.

\subsection{Additional Benchmarks and Diagnostics}
\label{sec:aux}

Two further experiments appear in \appref{app:additional}, an auxiliary benchmark of multi-modal point prediction (\appref{app:aux}) and a stratification of the SEMF-derived predictor by its predicted uncertainty (\appref{app:uncert}). A separate simulation in \appref{app:smallsample} quantifies the labeled budget at which tuning the disagreement scale begins to narrow intervals relative to simpler fixed choices. \autoref{tab:aux} compares the gated source-wise base model with homogeneous XGBoost and two neural fusion baselines. Among the multi-modal fusion models, a neural baseline leads in point accuracy on SRED, Mercari, and IMDB-WIKI, while homogeneous XGBoost leads on Pawpularity; the image-only solo predictor attains the highest Pawpularity value overall. The source-wise model still improves over homogeneous XGBoost on three datasets and yields marginal-CQR intervals near the nominal level. \autoref{fig:uncertainty} replaces $d(x)$ with the ensemble width of the augmented SEMF decoder outputs. Global coverage is $0.952$, $0.942$, and $0.946$ across the three bins, while MC-width Mondrian coverage is $0.945$, $0.941$, and $0.953$. The differences are small relative to the descriptive variation across the five seeds; this is not an inferential comparison and remains a weak diagnostic.

\begin{figure}[!htbp]
  \centering
  \includegraphics[width=\linewidth]{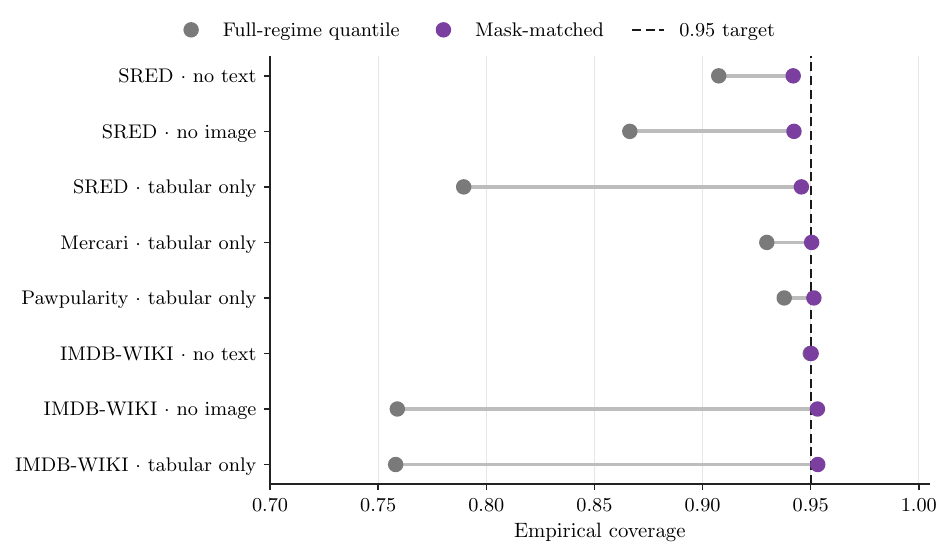}
  \caption{Missing-modality coverage repair across eight fixed-mask point-predictor experiments. Each row compares the same predictor under the full-regime quantile and mask-matched recalibration, the special case of Proposition~\ref{prop:mondrian} with a constant stratum label. Both rules share the same point center and constant radius within each mask, and recalibration restores coverage in harder masks by increasing that radius. Dots are five-seed mean coverages.}
  \label{fig:missing-dumbbell}
\end{figure}
\FloatBarrier

\section{Discussion}
\label{sec:discussion}

The experiments support the calibration mechanism, not a particular fusion architecture. A global quantile can allocate width inefficiently across regimes of source conflict or become inappropriate after a change in the test-time modality pattern. Treating these as calibration variables makes the uncertainty layer inspectable through the disagreeing sources, the observed availability regime, and the calibration examples supporting an interval. The two tools are complementary. Disagreement scaling targets efficiency under marginal coverage and uses a non-negative score to preserve the intended uncertainty ordering. It narrows low-disagreement cases and widens conflicted ones in most paired runs while leaving empirical PICP near nominal. Even on Mercari, where aggregate width is nearly unchanged, 12 of 15 runs narrow the intervals in the low-disagreement bin and widen them in the high-disagreement bin, and coverage in the high bin rises from $0.933$ to $0.940$.

Mondrian calibration is less smooth and can be less sample-efficient, but it gives an interpretable finite-stratum guarantee when the regime map is fixed and calibration and test examples are jointly exchangeable conditional on $\mathcal D_{\mathrm{pre}}$. The experiments with disagreement terciles apply it to a nontrivial partition. The fixed-mask stress tests instead use mask-matched recalibration, its special case with a constant stratum label, so larger residuals from the masked distribution produce appropriately larger intervals without claiming a partition over mixed regimes. The same modular pattern runs through all of these experiments. Source-wise predictors create a disagreement signal, and split-conformal calibration turns that signal into a scale that targets efficiency or into a stratum quantile that protects groups. The construction applies to any fixed base predictor, and the reported evidence covers gradient-boosting point, quantile, and stacked predictors plus SEMF-derived MC ensembles of decoder outputs.

The strict split protocol separates valid conformal calibration from post-hoc analysis. Pre-processing, encoders, early stopping, bin edges, missingness regimes, and scale selection all precede calibration. The cross-dataset runs follow this protocol, while the SRED missing-modality diagnostic reserves an evaluation half of the held-out predictions. This reduces sample efficiency but makes the empirical claims match the propositions more closely. For deployment, disagreement scaling is the lighter tool when average efficiency is primary. With roughly a hundred or fewer held-out labels, the simulation in \appref{app:smallsample} favors skipping the tuning split and either calibrating the unscaled score or fixing $\gamma=1$ in advance. Fixing $\gamma=1$ is not free, as it widens intervals by about a fifth relative to the marginal rule when disagreement carries no signal. Tuning the scale yields narrower intervals than the better of these fixed choices only from budgets of $200$ to $400$ labels onward, and only when neither is already close to the most efficient scale in the family. Mondrian calibration is preferable when coverage for a named modality pattern matters and that regime is represented in calibration. Reporting both separates efficiency from protection of specific regimes.

\section{Conclusion}
\label{sec:conclusion}

Multi-modal regression systems can fail their uncertainty estimates in two distinct structured ways: a global conformal quantile can allocate width inefficiently across regimes of source disagreement, and calibration can fail after a shift to a different test-time modality pattern. This paper turns that structure into two calibration variables. Source-wise disagreement drives a continuous disagreement-scaled conformal score that reallocates interval width while preserving marginal split-conformal validity. Across 60 paired headline comparisons, scaling matches or improves interval CRPS in 59 and width in 52 while keeping empirical coverage near $0.95$. Modality availability drives a Mondrian construction that protects pre-specified groups under the stated exchangeability conditions. The fixed-mask experiments use its mask-matched special case with a constant stratum label. Under those fixed masks, full-regime coverage falls as low as $0.76$, and mask-matched recalibration recovers up to $19.5$ percentage points. Together the two mechanisms form a model-agnostic wrapper for marginal efficiency and the protection of fixed groups.

\section{Limitations and Future Work}
\label{sec:limitations}

The finite-sample guarantees require every pipeline choice to be fixed before calibration, from the base and source predictors and their pre-processing to the availability rule, the stratum construction, and the disagreement scale. Calibration and test examples must then be exchangeable under that fixed pipeline. The cross-dataset benchmark follows this fit/tune/calibration/test protocol. The masking experiments zero the masked modality's feature block in both calibration and evaluation data, so they do not establish performance under naturally occurring or informative missingness. Interval CRPS evaluates a uniform distribution over the final interval rather than a full predictive law.

The scalar disagreement score gives every available source equal weight. It can miss cross-modal dependence and differences in source reliability, and correlated or substitute predictors can agree while all are wrong. Its behavior also depends on the number and redundancy of the available sources, so the continuous experiments deliberately keep one source set fixed and use a separate regime rule for missing modalities.

Sample size limits how flexible either calibration rule can be. At $\alpha=0.05$, a finite split-conformal quantile requires at least $19$ calibration examples, and Mondrian calibration requires that many in every protected stratum. Fixing $\gamma=1$ on a pre-specified standardized disagreement scale removes the dedicated step of selecting the scale, though not the tuning required by the base predictor or other data-driven choices. The simulation in \appref{app:smallsample} quantifies the resulting trade-off at small calibration budgets. Richer monotone, spline, or GAM-type scales need sufficient independent tuning data unless fixed in advance.

Several directions for future work follow. Cross-fitting could recover part of the sample efficiency that the strict split protocol gives up. Weighting the disagreement by source reliability would relax the equal treatment of sources, and systematic studies of $K$ and of source redundancy would clarify how the score behaves as the set of sources changes. Evaluations under naturally occurring missingness and under corruption at deployment would probe the regime tools beyond the fixed masks studied here.

\noindent\textbf{Acknowledgments.}
We thank Juraj Bodik, whose thorough feedback shaped some of the analyses, and Marc-Olivier Boldi and Valérie Chavez-Demoulin, whose careful reading and constructive discussions improved the paper throughout. We also gratefully acknowledge the HEC Research Fund, whose support made the computations on the DCSR clusters of the University of Lausanne possible.

\noindent\textbf{Code and data availability.} \url{https://unco3892.github.io/modality-aware-conformal}.

\bibliography{bibliography}

@inproceedings{chen2016xgboost,
  title     = {{XGBoost}: A Scalable Tree Boosting System},
  author    = {Chen, Tianqi and Guestrin, Carlos},
  booktitle = {Proceedings of the 22nd {ACM} {SIGKDD} International Conference on
               Knowledge Discovery and Data Mining},
  pages     = {785--794},
  year      = {2016},
  doi       = {10.1145/2939672.2939785}
}

@article{gao2020survey,
  title     = {A Survey on Deep Learning for Multimodal Data Fusion},
  author    = {Gao, Jing and Li, Peng and Chen, Zhikui and Zhang, Jianing},
  journal   = {Neural Computation},
  volume    = {32},
  number    = {5},
  pages     = {829--864},
  year      = {2020},
  publisher = {MIT Press},
  doi       = {10.1162/neco_a_01273}
}

@inproceedings{azizi2025semf,
  title={{SEMF}: Supervised Expectation-Maximization Framework for Predicting Intervals},
  author={Azizi, Ilia and Boldi, Marc-Olivier and Chavez-Demoulin, Val{\'e}rie},
  booktitle={Proceedings of the Fourteenth Symposium on Conformal and Probabilistic Prediction with Applications},
  series={Proceedings of Machine Learning Research},
  volume={266},
  pages={250--281},
  year={2025},
  publisher={PMLR},
  url={https://proceedings.mlr.press/v266/azizi25a.html}
}

@article{azizi2022sred,
  title={Improving Real Estate Rental Estimations with Visual Data},
  author={Azizi, Ilia and Rudnytskyi, Iegor},
  journal={Big Data and Cognitive Computing},
  volume={6},
  number={3},
  pages={96},
  year={2022},
  doi={10.3390/bdcc6030096}
}

@book{vovk2005algorithmic,
  title={Algorithmic Learning in a Random World},
  author={Vovk, Vladimir and Gammerman, Alexander and Shafer, Glenn},
  publisher={Springer},
  year={2005}
}

@article{shafer2008tutorial,
  title={A Tutorial on Conformal Prediction},
  author={Shafer, Glenn and Vovk, Vladimir},
  journal={Journal of Machine Learning Research},
  volume={9},
  pages={371--421},
  year={2008}
}

@inproceedings{papadopoulos2002inductive,
  title={Inductive Confidence Machines for Regression},
  author={Papadopoulos, Harris and Proedrou, Kostas and Vovk, Volodya and Gammerman, Alex},
  booktitle={Machine Learning: ECML 2002},
  series={Lecture Notes in Computer Science},
  volume={2430},
  pages={345--356},
  year={2002},
  publisher={Springer},
  doi={10.1007/3-540-36755-1_29}
}

@article{angelopoulos2023gentle,
  title={Conformal Prediction: A Gentle Introduction},
  author={Angelopoulos, Anastasios N. and Bates, Stephen},
  journal={Foundations and Trends in Machine Learning},
  volume={16},
  number={4},
  pages={494--591},
  year={2023},
  doi={10.1561/2200000101}
}

@article{lei2018distribution,
  title={Distribution-Free Predictive Inference for Regression},
  author={Lei, Jing and G'Sell, Max and Rinaldo, Alessandro and Tibshirani, Ryan J. and Wasserman, Larry},
  journal={Journal of the American Statistical Association},
  volume={113},
  number={523},
  pages={1094--1111},
  year={2018},
  doi={10.1080/01621459.2017.1307116}
}

@inproceedings{romano2019cqr,
  title={Conformalized Quantile Regression},
  author={Romano, Yaniv and Patterson, Evan and Cand{\`e}s, Emmanuel J.},
  booktitle={Advances in Neural Information Processing Systems},
  volume={32},
  pages={3538--3548},
  year={2019},
  url={https://papers.neurips.cc/paper/8613-conformalized-quantile-regression}
}

@inproceedings{vovk2012conditional,
  title={Conditional Validity of Inductive Conformal Predictors},
  author={Vovk, Vladimir},
  booktitle={Proceedings of the Asian Conference on Machine Learning},
  series={Proceedings of Machine Learning Research},
  volume={25},
  pages={475--490},
  year={2012},
  publisher={PMLR},
  url={https://proceedings.mlr.press/v25/vovk12.html}
}

@article{barber2021limits,
  title={The Limits of Distribution-Free Conditional Predictive Inference},
  author={Barber, Rina Foygel and Cand{\`e}s, Emmanuel J. and Ramdas, Aaditya and Tibshirani, Ryan J.},
  journal={Information and Inference: A Journal of the IMA},
  volume={10},
  number={2},
  pages={455--482},
  year={2021},
  doi={10.1093/imaiai/iaaa017}
}

@inproceedings{bellotti2021approximation,
  title={Approximation to Object Conditional Validity with Inductive Conformal Predictors},
  author={Bellotti, Anthony},
  booktitle={Proceedings of the Tenth Symposium on Conformal and Probabilistic Prediction and Applications},
  series={Proceedings of Machine Learning Research},
  volume={152},
  pages={4--23},
  year={2021},
  publisher={PMLR},
  url={https://proceedings.mlr.press/v152/bellotti21a.html}
}

@inproceedings{bostrom2020mondrian,
  title={Mondrian Conformal Regressors},
  author={Bostr{\"o}m, Henrik and Johansson, Ulf},
  booktitle={Proceedings of the Ninth Symposium on Conformal and Probabilistic Prediction and Applications},
  series={Proceedings of Machine Learning Research},
  volume={128},
  pages={114--133},
  year={2020},
  publisher={PMLR},
  url={https://proceedings.mlr.press/v128/bostrom20a.html}
}

@inproceedings{zaffran2023missing,
  title={Conformal Prediction with Missing Values},
  author={Zaffran, Margaux and Dieuleveut, Aymeric and Josse, Julie and Romano, Yaniv},
  booktitle={Proceedings of the 40th International Conference on Machine Learning},
  series={Proceedings of Machine Learning Research},
  volume={202},
  pages={40578--40604},
  year={2023},
  publisher={PMLR},
  url={https://proceedings.mlr.press/v202/zaffran23a.html}
}

@inproceedings{fan2026maskconditional,
  title={Mask-Conditional Conformal Prediction: Valid Uncertainty For All Missing Data Mechanisms},
  author={Fan, Jiarong and Park, Juhyun and Vo, Thi Phuong Thuy and Brunel, Nicolas J.-B.},
  booktitle={The 29th International Conference on Artificial Intelligence and Statistics},
  year={2026},
  url={https://openreview.net/forum?id=FiQAQSTXZn}

}

@inproceedings{johansson2024reject,
  title={Conformal Regression with Reject Option},
  author={Johansson, Ulf and S{\"o}nstr{\"o}d, Cecilia and Bostr{\"o}m, Henrik},
  booktitle={Proceedings of the Thirteenth Symposium on Conformal and Probabilistic Prediction with Applications},
  series={Proceedings of Machine Learning Research},
  volume={230},
  pages={277--294},
  year={2024},
  publisher={PMLR},
  url={https://proceedings.mlr.press/v230/johansson24a.html}
}

@inproceedings{tibshirani2019covariate,
  title={Conformal Prediction Under Covariate Shift},
  author={Tibshirani, Ryan J. and Barber, Rina Foygel and Cand{\`e}s, Emmanuel J. and Ramdas, Aaditya},
  booktitle={Advances in Neural Information Processing Systems},
  volume={32},
  year={2019},
  url={http://papers.neurips.cc/paper/8522-conformal-prediction-under-covariate-shift}
}

@inproceedings{papadopoulos2008normalized,
  title={Normalized Nonconformity Measures for Regression Conformal Prediction},
  author={Papadopoulos, Harris and Gammerman, Alex and Vovk, Vladimir},
  booktitle={Proceedings of the {IASTED} International Conference on Artificial Intelligence and Applications},
  pages={64--69},
  year={2008},
  publisher={ACTA Press}
}

@article{guan2023localized,
  title={Localized Conformal Prediction: A Generalized Inference Framework for Conformal Prediction},
  author={Guan, Leying},
  journal={Biometrika},
  volume={110},
  number={1},
  pages={33--50},
  year={2023},
  doi={10.1093/biomet/asac040}
}

@inproceedings{colombo2023localcp,
  title={On Training Locally Adaptive {CP}},
  author={Colombo, Nicolo},
  booktitle={Proceedings of the Twelfth Symposium on Conformal and Probabilistic Prediction with Applications},
  series={Proceedings of Machine Learning Research},
  volume={204},
  pages={384--398},
  year={2023},
  publisher={PMLR},
  url={https://proceedings.mlr.press/v204/colombo23a.html}
}

@article{izbicki2022hpd,
  title={{CD}-split and {HPD}-split: Efficient Conformal Regions in High Dimensions},
  author={Izbicki, Rafael and Shimizu, Gilson and Stern, Rafael Bassi},
  journal={Journal of Machine Learning Research},
  volume={23},
  number={87},
  pages={1--32},
  year={2022},
  url={https://jmlr.org/papers/v23/20-797.html}
}

@article{bian2023training,
  title={Training-Conditional Coverage for Distribution-Free Predictive Inference},
  author={Bian, Michael and Barber, Rina Foygel},
  journal={Electronic Journal of Statistics},
  volume={17},
  number={2},
  pages={2044--2066},
  year={2023},
  doi={10.1214/23-EJS2145}
}

@inproceedings{vovk2017cpd,
  title={Nonparametric Predictive Distributions Based on Conformal Prediction},
  author={Vovk, Vladimir and Shen, Jieli and Manokhin, Valery and Xie, Min-ge},
  booktitle={Proceedings of the Sixth Workshop on Conformal and Probabilistic Prediction and Applications},
  series={Proceedings of Machine Learning Research},
  volume={60},
  pages={82--102},
  year={2017},
  publisher={PMLR},
  url={https://proceedings.mlr.press/v60/vovk17a.html}
}

@article{vovk2020efficientcps,
  title={Computationally Efficient Versions of Conformal Predictive Distributions},
  author={Vovk, Vladimir and Petej, Ivan and Nouretdinov, Ilia and Manokhin, Valery and Gammerman, Alex},
  journal={Neurocomputing},
  volume={397},
  pages={292--308},
  year={2020},
  doi={10.1016/j.neucom.2019.10.110}
}

@inproceedings{bostrom2021mondrianpd,
  title={Mondrian Conformal Predictive Distributions},
  author={Bostr{\"o}m, Henrik and Johansson, Ulf and L{\"o}fstr{\"o}m, Tuwe},
  booktitle={Proceedings of the Tenth Symposium on Conformal and Probabilistic Prediction and Applications},
  series={Proceedings of Machine Learning Research},
  volume={152},
  pages={24--38},
  year={2021},
  publisher={PMLR},
  url={https://proceedings.mlr.press/v152/bostrom21a.html}
}

@inproceedings{lakshminarayanan2017deep,
  title={Simple and Scalable Predictive Uncertainty Estimation Using Deep Ensembles},
  author={Lakshminarayanan, Balaji and Pritzel, Alexander and Blundell, Charles},
  booktitle={Advances in Neural Information Processing Systems},
  volume={30},
  year={2017},
  url={https://papers.neurips.cc/paper/7219-simple-and-scalable-predictive-uncertainty-estimation-using-deep-ensembles}
}

@article{gneiting2007calibration,
  title={Probabilistic Forecasts, Calibration and Sharpness},
  author={Gneiting, Tilmann and Balabdaoui, Fadoua and Raftery, Adrian E.},
  journal={Journal of the Royal Statistical Society: Series B (Statistical Methodology)},
  volume={69},
  number={2},
  pages={243--268},
  year={2007},
  doi={10.1111/j.1467-9868.2007.00587.x}
}

@article{baltrusaitis2019survey,
  title={Multimodal Machine Learning: A Survey and Taxonomy},
  author={Baltru{\v{s}}aitis, Tadas and Ahuja, Chaitanya and Morency, Louis-Philippe},
  journal={IEEE Transactions on Pattern Analysis and Machine Intelligence},
  volume={41},
  number={2},
  pages={423--443},
  year={2019},
  doi={10.1109/TPAMI.2018.2798607}
}

@inproceedings{liang2024multimodal,
  title={Multimodal learning without labeled multimodal data: Guarantees and applications},
  author={Liang, Paul Pu and Ling, Chun Kai and Cheng, Yun and Obolenskiy, Alexander and Liu, Yudong and Pandey, Rohan and Wilf, Alex and Morency, Louis-Philippe and Salakhutdinov, Russ},
  booktitle={International Conference on Learning Representations},
  year={2024},
  url={https://openreview.net/forum?id=BrjLHbqiYs}
}

@article{wu2026missingmodalitysurvey,
  title={Deep Multimodal Learning with Missing Modality: A Survey},
  author={Wu, Renjie and Wang, Hu and Chen, Hsiang-Ting and Carneiro, Gustavo},
  journal={Transactions on Machine Learning Research},
  year={2026},
  url={https://openreview.net/forum?id=tc7RFcx4hT}
}

@inproceedings{wu2018mvae,
  title={Multimodal Generative Models for Scalable Weakly-Supervised Learning},
  author={Wu, Mike and Goodman, Noah},
  booktitle={Advances in Neural Information Processing Systems},
  volume={31},
  year={2018},
  url={https://papers.neurips.cc/paper/7801-multimodal-generative-models-for-scalable-weakly-supervised-learning}
}

@article{choi2019embracenet,
  title={{EmbraceNet}: A Robust Deep Learning Architecture for Multimodal Classification},
  author={Choi, Jun-Ho and Lee, Jong-Seok},
  journal={Information Fusion},
  volume={51},
  pages={259--270},
  year={2019},
  doi={10.1016/j.inffus.2019.02.010}
}

@inproceedings{ma2021smil,
  title={{SMIL}: Multimodal Learning with Severely Missing Modality},
  author={Ma, Mengmeng and Ren, Jian and Zhao, Long and Tulyakov, Sergey and Wu, Cathy and Peng, Xi},
  booktitle={Proceedings of the AAAI Conference on Artificial Intelligence},
  volume={35},
  pages={2302--2310},
  year={2021},
  url={https://ojs.aaai.org/index.php/AAAI/article/view/16330}
}

@inproceedings{wang2023shaspec,
  title={Multi-Modal Learning With Missing Modality via Shared-Specific Feature Modelling},
  author={Wang, Hu and Chen, Yuanhong and Ma, Congbo and Avery, Jodie and Hull, Louise and Carneiro, Gustavo},
  booktitle={Proceedings of the IEEE/CVF Conference on Computer Vision and Pattern Recognition (CVPR)},
  pages={15878--15887},
  year={2023}
}

@misc{bose2024multimodalconformal,
  title={Conformal Prediction for Multimodal Regression},
  author={Bose, Alexis and Ethier, Jonathan and Guinand, Paul},
  year={2024},
  eprint={2410.19653},
  archivePrefix={arXiv},
  primaryClass={cs.LG},
  doi={10.48550/arXiv.2410.19653},
  url={https://arxiv.org/abs/2410.19653}
}

@article{aas2009paircopula,
  title={Pair-Copula Constructions of Multiple Dependence},
  author={Aas, Kjersti and Czado, Claudia and Frigessi, Arnoldo and Bakken, Henrik},
  journal={Insurance: Mathematics and Economics},
  volume={44},
  number={2},
  pages={182--198},
  year={2009},
  doi={10.1016/j.insmatheco.2007.02.001}
}

@article{papamakarios2021flows,
  title={Normalizing Flows for Probabilistic Modeling and Inference},
  author={Papamakarios, George and Nalisnick, Eric and Rezende, Danilo Jimenez and Mohamed, Shakir and Lakshminarayanan, Balaji},
  journal={Journal of Machine Learning Research},
  volume={22},
  number={57},
  pages={1--64},
  year={2021},
  url={https://jmlr.org/papers/v22/19-1028.html}
}

@inproceedings{azizi2026clear,
 author = {Azizi, Ilia and Bodik, Juraj and Heiss, Jakob M. and Yu, Bin },
 booktitle = {International Conference on Learning Representations},
 editor = {C. Vondrick and B. Hariharan and C. Raffel and L. Pinto and D. Yang and A. Faust},
 pages = {157741--157813},
 title = {{CLEAR}: Calibrated Learning for Epistemic and Aleatoric Risk},
 url = {https://proceedings.iclr.cc/paper_files/paper/2026/file/fff035a88e93fff416f66f8a52dd9067-Paper-Conference.pdf},
 year = {2026}
}

@article{rothe2018deep,
  title     = {Deep Expectation of Real and Apparent Age from a Single Image Without Facial Landmarks},
  author    = {Rothe, Rasmus and Timofte, Radu and Van Gool, Luc},
  journal   = {International Journal of Computer Vision},
  volume    = {126},
  number    = {2--4},
  pages     = {144--157},
  year      = {2018},
  publisher = {Springer},
  doi       = {10.1007/s11263-016-0940-3}
}

@misc{mercari2018kaggle,
  title={Mercari Price Suggestion Challenge},
  author={{Mercari, Inc.}},
  year={2018},
  howpublished={Kaggle competition},
  url={https://www.kaggle.com/c/mercari-price-suggestion-challenge}
}

@misc{pawpularity2021kaggle,
  title={{PetFinder.my} - Pawpularity Contest},
  author={{PetFinder.my}},
  year={2021},
  howpublished={Kaggle competition},
  url={https://www.kaggle.com/c/petfinder-pawpularity-score}
}

@article{wang2022e5,
  title   = {Multilingual {E5} Text Embeddings: A Technical Report},
  author  = {Wang, Liang and Yang, Nan and Huang, Xiaolong and
             Yang, Linjun and Majumder, Rangan and Wei, Furu},
  journal = {arXiv preprint arXiv:2402.05672},
  year    = {2024},
  doi     = {10.48550/arXiv.2402.05672},
  url     = {https://arxiv.org/abs/2402.05672}
}

@inproceedings{reimers2019sbert,
  title     = {Sentence-{BERT}: Sentence Embeddings using {S}iamese {BERT}-Networks},
  author    = {Reimers, Nils and Gurevych, Iryna},
  booktitle = {Proceedings of the 2019 Conference on Empirical Methods in Natural
               Language Processing and the 9th International Joint Conference
               on Natural Language Processing (EMNLP-IJCNLP)},
  pages     = {3982--3992},
  year      = {2019},
  doi       = {10.18653/v1/D19-1410}
}

@inproceedings{reimers2020multilingual,
  title     = {Making Monolingual Sentence Embeddings Multilingual using
               Knowledge Distillation},
  author    = {Reimers, Nils and Gurevych, Iryna},
  booktitle = {Proceedings of the 2020 Conference on Empirical Methods in
               Natural Language Processing (EMNLP)},
  pages     = {4512--4525},
  year      = {2020},
  doi       = {10.18653/v1/2020.emnlp-main.365}
}

@article{oquab2024dinov2,
  title={{DINOv2}: Learning robust visual features without supervision},
  author={Oquab, Maxime and Darcet, Timoth{\'e}e and Moutakanni, Th{\'e}o and Vo, Huy and Szafraniec, Marc and Khalidov, Vasil and Fernandez, Pierre and Haziza, Daniel and Massa, Francisco and El-Nouby, Alaaeldin and others},
  journal={Transactions on Machine Learning Research},
  year={2024}
}


\appendix

\section{Metric Definitions}
\label{app:metrics}

For test intervals $[L_j,U_j]$, $j=1,\ldots,n_{\mathrm{te}}$, with labels $Y_j$ (the symbol $n_{\mathrm{te}}$ avoids a clash with the conformal rank $m$ of \autoref{sec:weighted}), define the observed target range $\Delta_Y=\max_jY_j-\min_jY_j$ and assume $\Delta_Y>0$ whenever a normalized width is reported. We use
\[
  \mathrm{PICP}
  =
  \frac{1}{n_{\mathrm{te}}}
  \sum_{j=1}^{n_{\mathrm{te}}}
  \mathds{1}\{Y_j\in[L_j,U_j]\},
\]
\[
  \mathrm{MPIW}
  =
  \frac{1}{n_{\mathrm{te}}}
  \sum_{j=1}^{n_{\mathrm{te}}}(U_j-L_j),
\]
and the normalized interval width
\[
  \mathrm{NIW}
  =
  \frac{\mathrm{MPIW}}{\Delta_Y},
\]
which coincides with the NMPIW metric of \citet{azizi2025semf}.

\paragraph{Normalized calibrated interval width.}
For the width, NCIW, and interval CRPS metrics, we verify that $L_j\leq U_j$ for all evaluated examples. For NCIW, let $\hat f_j$ be the scaling center. The center is a supplied model center when it lies in $[L_j,U_j]$ and is otherwise replaced by the interval midpoint $(L_j+U_j)/2$, so that $r_j^-=\hat f_j-L_j\geq 0$ and $r_j^+=U_j-\hat f_j\geq 0$. Let
\[
  c_{\mathrm{test\text{-}cal}}
  =
  \inf\left\{
    c\geq 0:
    \frac{1}{n_{\mathrm{te}}}\sum_{j=1}^{n_{\mathrm{te}}}
    \mathds{1}\{Y_j\in[\hat f_j-c r_j^-,\hat f_j+c r_j^+]\}
    \geq 1-\alpha
  \right\},
\]
then
\[
  \mathrm{NCIW}
  =
  c_{\mathrm{test\text{-}cal}}\,
  \frac{\mathrm{MPIW}}{\Delta_Y}.
\]
We adopt $\inf\emptyset=+\infty$ and define $\mathrm{NCIW}=+\infty$ if no finite factor reaches the target coverage. Every reported evaluation set has $\Delta_Y>0$ and a finite feasible factor.
The supplied center is the fitted point prediction for point-predictor intervals and the fitted median for the cross-dataset quantile predictors, while the source-wise quantile intervals of \autoref{tab:aux} use the interval midpoint. In the SRED per-predictor diagnostic it is the XGBoost point prediction for the XGB point rows, the mean of the 50 outputs of the augmented decoder for the Aug.\ SEMF rows, and the midpoint of the fitted XGBoost endpoint pair for the XGB quantile rows. The midpoint replacement above is then applied only if that supplied center falls outside the final interval.

\paragraph{Interval CRPS.}
Interval CRPS is the average closed-form CRPS of $\mathrm{Unif}[L_j,U_j]$, with limiting value $|Y_j-L_j|$ for intervals of zero width. For width $w=U-L>0$ and $t=(y-L)/w$, the per-example score is
\[
\mathrm{CRPS}(\mathrm{Unif}[L,U],y)=
\begin{cases}
L-y+w/3, & y<L,\\
w(t^2-t+1/3), & L\leq y\leq U,\\
y-U+w/3, & y>U.
\end{cases}
\]

\paragraph{Coefficient of determination.}
$R^2$ in \autoref{tab:aux} is the usual coefficient of determination computed on the test split.

\section{Reproducibility Details}
\label{app:repro}

\paragraph{Splits and seeds.}
All quantitative experiments on the four datasets use seeds $0,\ldots,4$ and $\alpha=0.05$. The score-level simulation of \appref{app:smallsample} instead draws its $2{,}000$ replications from one random number generator with seed $0$. The cross-dataset benchmarks split each development corpus into $65\%$ fit, $15\%$ tune, and $20\%$ calibration subsets, which yields the counts in \autoref{tab:datasets}. Test sets are fixed before these five seeded development splits.

For Mercari, fixed seed-0 randomization samples 50,000 labeled rows from the competition training file and partitions them into 80\% development and 20\% test sets. Missing title or description text is replaced with the empty string, with no further filtering or deduplication. The five experiment seeds then split the 40,000-row development portion into fit, tune, and calibration subsets. Pawpularity uses a fixed seed-0 80/20 permutation of the labeled Kaggle training file, yielding 1,983 test examples. IMDB-WIKI uses a deterministic 90/10 assignment obtained by hashing each photo path, yielding 3,922 test examples. An overlap check finds that 43 of these test rows (1.10\%) have a normalized non-empty name, after case folding and stripping, also present in the training split. For SRED, a corresponding check finds that 37 of 1,109 test rows (3.34\%) exactly match a training row on latitude, longitude, living space, room count, and recorded rent, with rent used only for this check and never as an input feature. Such near-duplicates, common in listing corpora, can mildly inflate absolute point accuracy and may weaken exchangeability at the row level if re-listed units form dependent clusters, so a grouped or deduplicated split is a natural robustness check left to future work.
The empirical disagreement terciles use linear interpolation between adjacent order statistics of the tuning split. Boundary assignment is defined in \autoref{sec:mondrian}.

\begin{table}[!htbp]
\centering
\small
\setlength{\tabcolsep}{4pt}
\caption{Datasets, targets, and split sizes for the cross-dataset multi-modal benchmark.}
\label{tab:datasets}
\begin{tabular*}{\linewidth}{@{\extracolsep{\fill}}llrrrr}
\toprule
Dataset & Target & Fit & Tune & Cal. & Test \\
\midrule
SRED & $\log(\mathrm{rent})$ & 6,498 & 1,499 & 1,999 & 1,109 \\
Mercari & $\log(\mathrm{price}+1)$ & 26,000 & 6,000 & 8,000 & 10,000 \\
Pawpularity & score & 5,154 & 1,189 & 1,586 & 1,983 \\
IMDB-WIKI & age & 22,238 & 5,132 & 6,843 & 3,922 \\
\bottomrule
\end{tabular*}

\end{table}

\paragraph{Selection of the disagreement scale.}
The cross-dataset benchmark constructs its candidate values of the canonical parameter $\gamma$ from the source grids
\[
  a_0\in\{10^{-3},10^{-2},0.1,0.5,1,3\},
  \qquad
  a_1\in\{0,0.25,0.5,1,2,4,8,16\}.
\]
It takes the sorted unique ratios $\gamma=a_1/a_0$, evaluates the normalized representative $a_\gamma(x)=\sqrt{1+\gamma d(x)^2}$, and minimizes the conformal quantile of the scaled scores multiplied by the mean scale, both computed on the tuning split. This objective is a width proxy. Candidate order is deterministic, and strict objective improvement retains the smallest $\gamma$ on an exact tie. The disagreement-scaled tuner and its calibration step use the clipped score of \autoref{eq:scaled_score}. The objective does not enter the validity proof, and the selected scale and final quantiles are retained for every run. In the cross-dataset benchmark, disagreement is standardized before grid search by the interquartile range of the tuning split, with fallbacks to the standard deviation and then to a unit scale. This standardization belongs to $\mathcal D_{\mathrm{pre}}$. In the units of unstandardized disagreement, the selected coefficient is $\gamma/c_d^2$, where $c_d$ is that scale. The SRED diagnostic retains the two-parameter grid with the same $a_0$ values and $a_1\in\{0,0.5,1,3,10,30\}$ on unstandardized disagreement. Its fixed iteration order likewise resolves exact objective ties deterministically.
The counts of strict improvements and exact matches compare the underlying per-run values at full precision. ``Exact tie'' means numerical equality before rounding, with no additional floating-point tolerance.
The XGBoost-learned scale uses 200 trees, maximum depth 3, learning rate $0.05$, and the squared error loss. It is fitted on the fused covariates of the tuning split to $\log\max\{r,10^{-6}\}$, where $r$ is the non-negative tuning residual or CQR score. Its exponentiated predictions are normalized by their median on the tuning split and frozen before calibration. Calibration scores are divided by this scale, and the test correction is multiplied by it.

\paragraph{Binned-scaled variant.}
The secondary binned-scaled variant of \autoref{sec:setup} replaces the continuous scale with a monotone piecewise-constant scale on the same standardized disagreement. Its three bins use the terciles of the tuning split with the edge convention of \autoref{sec:mondrian}. The candidate multiplier vectors are the identity $(1,1,1)$ together with $(1,\sqrt{r},r)$ for $r\in\{1.1, 1.25, 1.5, 2, 3, 4, 6\}$, so every candidate is monotone by construction. For each candidate, the tuner computes the conformal quantile of the clipped scores divided by the binned scale and minimizes the mean of the base interval width plus twice that quantile times the scale, both on the tuning split. For point predictors the base width is zero, and this objective is proportional to the width proxy of the continuous tuner. Candidates are evaluated with the identity first and then in increasing ratio order, strict improvement retains the earlier candidate on an exact tie, and the selected edges and multipliers are frozen before calibration. Calibration then applies the clipped score and interval construction of \autoref{sec:weighted} with the selected binned scale in place of $a_\gamma$. Across the 60 paired cross-dataset headline runs, the binned variant matches or improves marginal CRPS in 53 comparisons and width in 50.

\paragraph{Base predictors.}
Cross-dataset XGBoost point predictors use 700 trees, maximum depth 6, learning rate $0.04$, subsampling and column subsampling of $0.85$, and early stopping (50 rounds) monitored on the tuning split. Cross-dataset XGBoost quantile models use 500 trees at levels $\alpha/2$, $1/2$, and $1-\alpha/2$ with the same depth and learning rate and no early stopping. The per-source auxiliary predictors $g_k$ in this benchmark are per-source XGBoost point regressors with the same 700-tree settings and early stopping on the tuning split. The ridge stacker of \autoref{sec:main-results} combines these per-source point predictions with ridge regression at regularization strength $1.0$, fitted with an intercept on the tuning split.

The SRED per-predictor diagnostic uses different fixed base models. Its XGBoost point base uses 2,000 trees, maximum depth 6, learning rate $0.03$, subsampling and column subsampling of $0.8$, and 50-round early stopping on the trailing 10\% of the fitting split. Its two XGBoost quantile endpoints use 400 trees, maximum depth 6, learning rate $0.05$, subsampling and column subsampling of $0.8$, and no early stopping. The five predictors of the diagnostic's source blocks use 300 trees, maximum depth 6, learning rate $0.05$, subsampling and column subsampling of $0.8$, and 30-round early stopping on a fixed 10\% holdout of the fitting split.

In \autoref{tab:aux}, the source-wise, solo, and concatenated predictors follow the auxiliary configuration below. The cross-dataset missing-modality experiment uses the 700-tree configuration of the point model above.

\paragraph{Auxiliary benchmark models.}
The source-wise predictors of \autoref{tab:aux} are an XGBoost quantile triple on the tabular block and one multi-layer perceptron (MLP) trained with the pinball loss per embedded source. The solo rows evaluate these same fitted predictors alone, and the concatenated baseline applies the triple's configuration to the concatenated blocks. The triple fits one model per level $\alpha/2$, $1/2$, and $1-\alpha/2$ with 400 trees on all four datasets, maximum depth 6, learning rate $0.05$, and subsampling and column subsampling of $0.9$. Each source MLP maps its embedding through 256 and 128 GELU units with dropout $0.1$ to the three levels and minimizes the multi-quantile pinball loss with AdamW at learning rate $10^{-3}$, weight decay $10^{-4}$, and batch size 256 for at most 50 epochs, with early stopping after 8 epochs without improvement at threshold $10^{-5}$ on a seeded 10\% split of its fitting fold and restoration of the best weights. The gate maps the concatenated blocks through 128 GELU units with dropout $0.1$ to per-example softmax weights over the source predictors, shared across the three levels. Each source triple is sorted at prediction time, the gate output at level $\tau$ is the convex combination $\sum_k w_k(x)\,\hat q_{k,\tau}(x)$, and the fused triple is sorted again before calibration and evaluation. The gate trains on the tuning split only with the same loss, optimizer settings, and batch size, at most 80 epochs for SRED and Pawpularity and 60 for Mercari and IMDB-WIKI, and patience 12 at threshold $10^{-5}$ on a seeded 15\% split of the tuning fold, with the best validation weights restored. For the source-wise, solo, concatenated, and gate models, the target is standardized with statistics of the fitting fold, and all metrics are reported on the original scale.

\paragraph{Auxiliary neural baselines.}
Both fusion baselines minimize mean squared error on the standardized target. The MLP on concatenated features applies three blocks of layer normalization, a linear map, GELU, and dropout $0.2$ with widths 512, 256, and 128 before a linear output, at learning rate $10^{-3}$. The cross-attention baseline encodes each source into a 64-dimensional token through a per-source linear encoder with layer normalization, GELU, and dropout, applies two pre-norm transformer layers with four-head self-attention, GELU feed-forward activation, and feed-forward expansion 2, mean-pools the tokens, and finishes with a head that applies layer normalization, one 64-unit GELU layer with dropout, and a linear output, at learning rate $5\times10^{-4}$. Each listed block applies its operations in the order given, and every dropout in both baselines, including attention and feed-forward dropout, uses probability $0.2$. Both use AdamW with weight decay $10^{-4}$ and batch size 128 for at most 100 epochs under a cosine schedule with 5\% linear warmup and fresh per-epoch shuffling, with early stopping at patience 10 and threshold $10^{-6}$ on the internal split described below and restoration of the best weights. The experiment seed drives initialization, shuffling, and every internal split.

\paragraph{Frozen encoders.}
Text and image blocks are frozen pretrained embeddings, mean-pooled within each modality. We use multilingual-e5-large \citep{wang2022e5}, paraphrase-multilingual-MiniLM-L12-v2 \citep{reimers2019sbert,reimers2020multilingual}, and DINOv2 \citep{oquab2024dinov2}. Cross-dataset SRED embeds the listing header and description with multilingual-e5-large and the photo montage and satellite views with DINOv2 ViT-B/14. The SRED SEMF-derived diagnostics (\autoref{tab:predagn}, \autoref{tab:regime}, \autoref{fig:regime}, \autoref{fig:uncertainty}) instead build their 36-dimensional representation with paraphrase-multilingual-MiniLM-L12-v2 and DINOv2 ViT-S/14. Mercari embeds the item name and description with the same MiniLM model, Pawpularity embeds photos with DINOv2 ViT-S/14, and IMDB-WIKI embeds names with MiniLM and faces with DINOv2 ViT-S/14. Tabular columns are standardized and categorical variables one-hot encoded using statistics of the fitting split in the XGBoost pipelines. The neural baselines of \autoref{tab:aux} pool the fit and tuning examples for pre-processing, target standardization, and parameter training, then use a seeded random 90/10 internal train/validation split of that pool for early stopping; the calibration and test splits remain untouched. All runs use the per-dataset encoders stated above. A source is a fixed feature block with its own predictor. The cross-dataset source sets in \autoref{tab:cross-predagn} have $K=3$ for SRED and IMDB-WIKI and $K=2$ for Mercari and Pawpularity. The SRED SEMF-derived diagnostic uses the five field blocks of \appref{app:semf}. The IMDB-WIKI subset keeps photos with a single face and valid metadata. Its tabular features are gender, photo year, and the face detection score, with neither date of birth nor identity codes. Nearly every person contributes one photo (37,903 names over 38,135 rows).

\paragraph{SRED SEMF-derived diagnostic.}
The diagnostic of \autoref{tab:predagn} uses precomputed SEMF predictions from 8,481 fit examples, a 1,497-example validation set, and 1,109 test examples. The pipeline drops 18 rows duplicated in every feature and the target before splitting the remaining 9,978 training rows $85/15$, which explains the difference from \autoref{tab:datasets}. For experiment seed $s$, a pseudorandom split with seed $s+10{,}003$ divides the validation set into tuning and calibration halves.

Its 36 features comprise four tabular variables plus eight principal component analysis (PCA) components per embedded text or image field, with PCA fitted on all 9,996 training rows before the duplicate removal and the 85/15 split described above. Upstream training statistics standardize the log-rent target. MPIW and interval CRPS are therefore in standardized units, PICP is unchanged, and NCIW is dimensionless.

Fitted XGBoost-quantile pairs are not rearranged before calibration. Crossed pairs remain valid CQR scores, and all evaluated endpoints are ordered. This diagnostic is therefore a departure from the construction with ordered endpoints in \autoref{sec:cp-bg} and Algorithm~\ref{algo:macc}. The cross-dataset pipeline instead sorts each fitted quantile triple before calibration, which gives weakly wider base endpoints than the pairwise rearrangement of \autoref{sec:cp-bg}.

Missing-modality tests zero the standardized modality block in calibration and test examples without adding an indicator. For seed $s$, a fixed permutation seeded by $s+210{,}517$ assigns 554 held-out examples to calibration and 555 to evaluation. This rule was fixed independently of labels and diagnostics.

\section{SEMF-derived Base Predictor}
\label{app:semf}

The SRED diagnostics use two base predictors derived from SEMF \citep{azizi2025semf}. The SEMF formulation represents the conditional law of $Y$ through per-source latent variables,
\begin{equation}
  p(y\mid x)
  =
  \int p_\theta(y\mid z_1,\ldots,z_K)
  \prod_{k=1}^{K} p_{\phi_k}(z_k\mid x^{(k)})\,
  dz_1\cdots dz_K,
  \qquad z=(z_1,\ldots,z_K).
  \label{eq:semf_factorization}
\end{equation}
Each source can have a distinct encoder family, while the decoder consumes a fused latent representation. In the SEMF formulation, ``source'', $K$, and $x^{(k)}$ denote encoder groups rather than the source blocks used for calibration in the main text. Likewise, $R$ and $s$ in this appendix count latent replications and response draws rather than denoting the Mondrian score function and scaled scores of Algorithm~\ref{algo:macc}. The 36-feature SRED representation uses nine consecutive encoder groups of four features. Calibration disagreement instead uses five input blocks: one tabular block, two text fields, and two image fields. The masks aggregate these blocks into tabular, text, and image availability regimes. This distinction does not affect the conformal wrapper.

The reported SEMF checkpoints use $R=10$ latent replications during training, $R_{\mathrm{infer}}=50$ at inference, at most 15 outer iterations, and outer early-stopping patience 5. Each group of four features has a two-dimensional latent with fixed standard deviation $0.1$. All SEMF encoder and decoder learners are 100-tree XGBoost regressors with learning rate $0.05$, subsampling and column subsampling of $0.8$, and no early stopping within the learners.

Training uses a generalized Expectation-Maximization (EM) style loop. On the first iteration the replication weights are uniform, $w_{i,r}=1/R$. On each subsequent iteration, writing $\phi'$ and $\theta'$ for parameters held fixed from the preceding iteration and drawing $z_{i,r}\sim p_{\phi'}(z\mid x_i)$ for $r=1,\ldots,R$, the E-step forms the self-normalized weights
\begin{equation}
  w_{i,r}
  =
  \frac{p_{\theta'}(y_i\mid z_{i,r})}
       {\sum_{t=1}^{R} p_{\theta'}(y_i\mid z_{i,t})},
  \label{eq:semf_weights}
\end{equation}
assuming a positive denominator. The M-step then refits the encoders and decoder with weighted supervised losses. We call this EM-style because these updates approximate rather than exactly maximize the likelihood.

At inference, SEMF first draws latents $z_r\sim p(z\mid x)$ and then responses $\widetilde Y_{r,s}\sim p(y\mid z_r)$ for $s=1,\ldots,R$ \citep[Section~2.4]{azizi2025semf}. For the SRED diagnostics, the response draw is omitted and the $R_{\mathrm{infer}}=50$ decoder outputs $f_\theta(z_r)$ are retained for each example. Their linearly interpolated empirical 2.5th and 97.5th percentiles form the base endpoints. These ensembles propagate only the latent draws through the decoder and are not samples from the full predictive law $p(y\mid x)$.

The augmented variant instead retains $f_{\theta,\mathrm{aug}}(x,z_r)$. Its decoder uses 1,000 XGBoost trees, maximum depth 6, learning rate $0.03$, subsampling and column subsampling of $0.8$, and 50-round early stopping on the validation split. This augmented decoder is fitted on the latent conditional means $\mathbb{E}[z\mid x]$ under the frozen encoders and is evaluated on the $R_{\mathrm{infer}}$ sampled latents, a train--inference mismatch that may contribute to the narrow decoder-output ensemble and the underdispersion reported below. The diagnostic across base predictors (\autoref{tab:predagn}, Aug.\ SEMF rows) and predicted-uncertainty stratification (\autoref{fig:uncertainty}) use this augmented decoder. The missing-modality diagnostic (\autoref{tab:regime}, \autoref{fig:regime}) uses the original decoder, which consumes only the fused latent sample.

With all inputs observed, the raw ensemble intervals are markedly underdispersed: across five seeds their empirical coverage is about $0.12$ for the augmented decoder and $0.19$ for the original decoder at nominal $0.95$. About $94\%$ of the calibrated augmented-SEMF width in \autoref{tab:predagn} is supplied by the conformal radius, so that row is close to a point construction on absolute residuals. The MC predictions are generated once without label access under a fixed seeded rule and belong to $\mathcal D_{\mathrm{pre}}$. We therefore report these analyses as diagnostics.

The reported SEMF-derived results use five checkpoints in which non-constant E-step weights are correctly aligned with the replicated M-step rows. All downstream analyses use these checkpoints.

\section{Additional Results and Diagnostics}
\label{app:additional}

\subsection{XGBoost-learned Scale Sensitivity Analysis}
\label{app:xgb-scale}

\begin{table}[!htbp]
\centering
\small
\setlength{\tabcolsep}{2.4pt}
\caption{Sensitivity analysis for an XGBoost-learned scale across four multi-modal datasets. Each dataset aggregates the same three base predictors and five seeds as \autoref{tab:cross-predagn}; entries are the mean $\pm$ one sample standard deviation over 15 predictor--seed runs. Marginal and disagreement-scaled rows are repeated for context. \textit{XGBoost-learned scale} is a specific normalized split conformal/CQR comparator. On the tuning split, a shallow XGBoost regressor maps fused features to the logarithm of the non-negative residual or CQR score after a $10^{-6}$ floor. Its exponentiated predictions are normalized by their median on the tuning split and frozen before calibration. Calibration scores are divided by this scale, and the test correction is multiplied by it. This comparator is not intended to represent learned normalization methods in general. MPIW and CRPS have dataset-specific target units, and NCIW is dimensionless.}
\label{tab:cross-predagn-sensitivity}
\begin{tabular*}{\linewidth}{@{\extracolsep{\fill}}llcccc}
\toprule
Dataset & Calibration & PICP & MPIW & NCIW & CRPS \\
\midrule
SRED & Marginal & $0.943\pm0.005$ & $0.714\pm0.136$ & $0.349\pm0.072$ & $0.103\pm0.020$ \\
 & Disagreement-scaled & $0.942\pm0.006$ & $0.689\pm0.128$ & $0.341\pm0.069$ & $0.102\pm0.020$ \\
 & \shortstack[l]{XGBoost-learned\\scale} & $0.941\pm0.006$ & $0.915\pm0.338$ & $0.452\pm0.156$ & $0.119\pm0.035$ \\
\cmidrule{1-6}
Mercari & Marginal & $0.950\pm0.002$ & $2.439\pm0.083$ & $0.326\pm0.011$ & $0.373\pm0.025$ \\
 & Disagreement-scaled & $0.950\pm0.002$ & $2.439\pm0.084$ & $0.326\pm0.011$ & $0.373\pm0.025$ \\
 & \shortstack[l]{XGBoost-learned\\scale} & $0.951\pm0.002$ & $2.635\pm0.386$ & $0.351\pm0.053$ & $0.387\pm0.044$ \\
\cmidrule{1-6}
Pawpularity & Marginal & $0.953\pm0.005$ & $84.426\pm4.608$ & $0.858\pm0.048$ & $12.791\pm2.123$ \\
 & Disagreement-scaled & $0.952\pm0.006$ & $83.020\pm5.599$ & $0.847\pm0.054$ & $12.717\pm2.188$ \\
 & \shortstack[l]{XGBoost-learned\\scale} & $0.953\pm0.006$ & $422.601\pm792.194$ & $4.320\pm8.032$ & $40.844\pm65.663$ \\
\cmidrule{1-6}
IMDB-WIKI & Marginal & $0.952\pm0.003$ & $40.616\pm4.799$ & $0.447\pm0.052$ & $6.489\pm1.294$ \\
 & Disagreement-scaled & $0.951\pm0.003$ & $40.464\pm4.650$ & $0.446\pm0.051$ & $6.478\pm1.286$ \\
 & \shortstack[l]{XGBoost-learned\\scale} & $0.952\pm0.003$ & $45.051\pm10.940$ & $0.495\pm0.118$ & $6.802\pm1.704$ \\
\bottomrule
\end{tabular*}

\end{table}
\FloatBarrier

\subsection{SRED Diagnostic Across Base Predictors}
\label{app:predagn}

\begin{table}[!htbp]
\centering
\small
\setlength{\tabcolsep}{3.5pt}
\caption{SRED calibration across base predictors at target coverage 0.95. Cells are five-seed means $\pm$ one sample standard deviation, reported as descriptive variation across runs rather than standard errors, confidence intervals, or inferential uncertainty. MPIW and CRPS use standardized log-rent units, while NCIW is dimensionless. Lower is better for MPIW, NCIW, and CRPS, while PICP should be near 0.95. In the calibration labels, ``dis.-scaled'' and ``dis.-Mondrian'' abbreviate disagreement-scaled and disagreement-Mondrian calibration. Bold marks the best value per base predictor for the metrics where lower is better. All rows share a precomputed representation whose PCA maps were fitted before the SEMF training data were divided into fitting and validation subsets, and SEMF-derived rows additionally reuse validation data from SEMF fitting. The table is therefore diagnostic rather than a formal proposition instance.}
\label{tab:predagn}
\begin{tabular*}{\linewidth}{@{\extracolsep{\fill}}llcccc}
\toprule
Base predictor & Calibration & PICP & MPIW & NCIW & CRPS \\
\midrule
XGB point & marginal & $0.951\pm0.007$ & $1.826\pm0.069$ & $0.248\pm0.006$ & $0.257\pm0.004$ \\
XGB point & dis.-Mondrian & $0.956\pm0.007$ & $1.860\pm0.139$ & $0.241\pm0.011$ & $0.256\pm0.007$ \\
XGB point & dis.-scaled & $0.953\pm0.010$ & $\mathbf{1.795\pm0.084}$ & $\mathbf{0.239\pm0.004}$ & $\mathbf{0.254\pm0.003}$ \\
\midrule
Aug. SEMF & marginal & $0.949\pm0.008$ & $1.814\pm0.081$ & $0.246\pm0.004$ & $0.257\pm0.004$ \\
Aug. SEMF & dis.-Mondrian & $0.952\pm0.008$ & $1.817\pm0.112$ & $0.245\pm0.011$ & $0.255\pm0.005$ \\
Aug. SEMF & dis.-scaled & $0.949\pm0.009$ & $\mathbf{1.741\pm0.080}$ & $\mathbf{0.239\pm0.007}$ & $\mathbf{0.252\pm0.004}$ \\
\midrule
XGB quantile & marginal CQR & $0.944\pm0.014$ & $2.205\pm0.172$ & $0.308\pm0.002$ & $0.321\pm0.007$ \\
XGB quantile & dis.-Mondrian & $0.944\pm0.012$ & $2.179\pm0.153$ & $0.304\pm0.004$ & $0.318\pm0.005$ \\
XGB quantile & dis.-scaled & $0.942\pm0.011$ & $\mathbf{2.112\pm0.128}$ & $\mathbf{0.298\pm0.002}$ & $\mathbf{0.316\pm0.004}$ \\
\bottomrule
\end{tabular*}

\end{table}

\autoref{tab:predagn} summarizes the SRED calibration diagnostic per base predictor. \appref{app:repro} details the precomputed 36-dimensional representation and splits. Disagreement-scaled split conformal/CQR gives the lowest mean NCIW and mean raw width for all three base predictors in this sweep, while disagreement-Mondrian gives the explicitly stratified validity mechanism. For the augmented SEMF base, scaling leaves mean PICP at $0.949$ while reducing MPIW from $1.814$ to $1.741$ and CRPS from $0.257$ to $0.252$.

\subsection{SRED Mask-matched Missing-modality Table}
\label{app:regime-table}

\autoref{tab:regime} reports the full per-regime coverage and width summarized by \autoref{fig:regime}.

\begin{table}[!htbp]
\centering
\caption{Held-out SRED SEMF-derived diagnostic coverage under test-time modality masking. Cells report the mean $\pm$ one sample standard deviation over the five seeds as descriptive variation across runs, not standard errors or confidence intervals. Global CQR uses the full-regime quantile. Mask-matched CQR applies the row's fixed mask to every calibration and evaluation example; it is the special case of Proposition~\ref{prop:mondrian} with a constant stratum label, so its quantile is the ordinary per-mask split-conformal quantile. MPIW is in standardized log-rent units. The held-out test predictions are split into calibration and evaluation subsets.}
\label{tab:regime}
\begin{tabular*}{\linewidth}{@{\extracolsep{\fill}}lcccc}
\toprule
Regime & Global PICP & Global MPIW & \shortstack[c]{Mask-matched\\PICP} & \shortstack[c]{Mask-matched\\MPIW} \\
\midrule
Full & $0.941\pm0.018$ & $3.053\pm0.147$ & $0.941\pm0.018$ & $3.053\pm0.147$ \\
No text & $0.938\pm0.015$ & $3.056\pm0.152$ & $0.942\pm0.015$ & $3.140\pm0.173$ \\
No image & $0.923\pm0.018$ & $3.050\pm0.148$ & $0.946\pm0.014$ & $3.393\pm0.107$ \\
Tabular only & $0.921\pm0.020$ & $3.045\pm0.157$ & $0.947\pm0.013$ & $3.472\pm0.167$ \\
\bottomrule
\end{tabular*}

\end{table}

\begin{figure}[!htbp]
  \centering
  \includegraphics[width=0.7\linewidth]{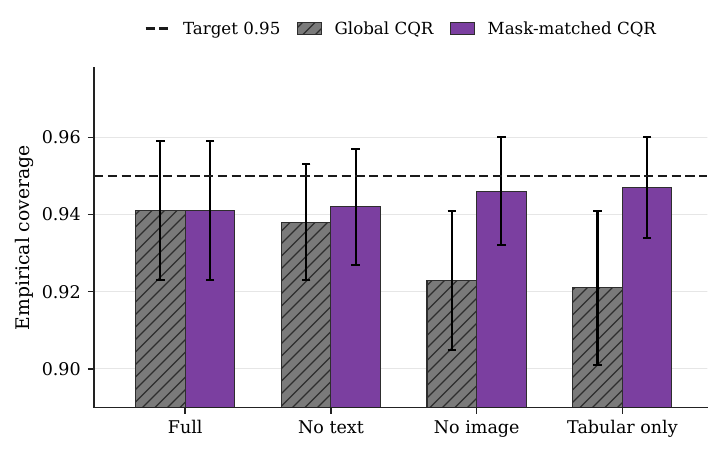}
  \caption{Global and mask-matched coverage under missing modalities, shown on a zoomed y-axis. Bars are five-seed means and whiskers span $\pm$ one sample standard deviation across seeds as descriptive variation across runs, not standard errors or confidence intervals. The largest global undercoverage appears in the tabular-only and no-image masks, while mask-matched recalibration stays closer to the 0.95 target. Each mask uses the special case of Proposition~\ref{prop:mondrian} with a constant stratum label.}
  \label{fig:regime}
\end{figure}
\FloatBarrier

\subsection{Strata of Predicted Uncertainty}
\label{app:uncert}

Bins of predicted uncertainty show only mild variation. For the augmented SEMF-derived predictor, global CQR coverage is $0.952\pm0.026$, $0.942\pm0.026$, and $0.946\pm0.014$ in the low, middle, and high bins of MC width. MC-width Mondrian calibration gives $0.945\pm0.020$, $0.941\pm0.031$, and $0.953\pm0.022$, respectively. The differences between bins and between methods are small relative to the descriptive variation across the five seeds; this is not an inferential comparison, so the figure is a weak diagnostic rather than evidence of a strong conditional repair. This analysis avoids the SEMF validation split for conformal calibration, but because the MC-width bin edges are estimated from the calibration half's predicted widths, it remains a post-hoc diagnostic rather than a formal instance of Proposition~\ref{prop:mondrian}.

\begin{figure}[!htbp]
  \centering
  \includegraphics[width=0.7\linewidth]{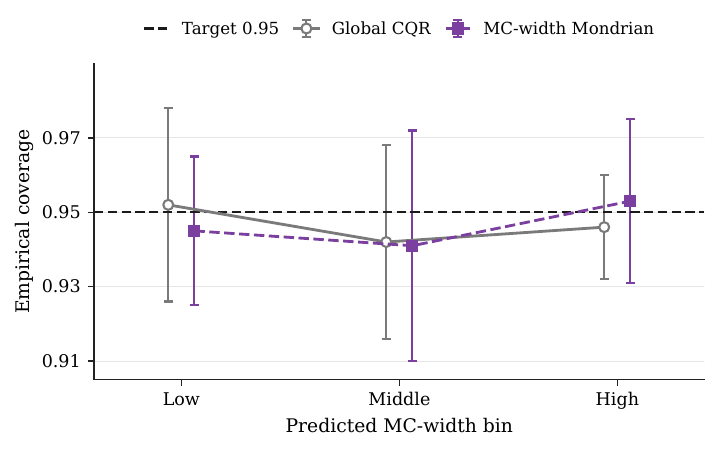}
  \caption{Diagnostic coverage on the held-out SRED test split across bins of MC ensemble width for the augmented SEMF-derived predictor, shown on a zoomed y-axis. Markers are five-seed means and whiskers span $\pm$ one sample standard deviation across seeds as descriptive variation across runs, not standard errors or confidence intervals. Both profiles remain near the target, and their differences are small relative to that variation; no inferential comparison is intended.}
  \label{fig:uncertainty}
\end{figure}
\FloatBarrier

\subsection{Additional Benchmarks of Multi-modal Point Prediction}
\label{app:aux}

\autoref{tab:aux} reports the additional multi-modal benchmark. Each dataset uses tabular features plus frozen text and/or image embeddings when available. The source-wise model trains per-source quantile predictors and fuses them with a gate trained on the tuning split. This gated model is distinct from the source-wise ensemble of \autoref{sec:main-results}, which uses fixed weights inversely proportional to RMSE on the tuning split. It is compared with homogeneous XGBoost on concatenated features and two neural network baselines. These are an MLP on concatenated features and a cross-attention network for multi-modal inputs.

\begin{table}[!htbp]
\centering
\scriptsize
\setlength{\tabcolsep}{2.4pt}
\caption{Additional multi-modal benchmarks. $R^2$ cells report the mean $\pm$ one sample standard deviation over five seeds as descriptive variation across runs, not standard errors or confidence intervals; PICP and NCIW are five-seed means. Source-wise and Concat $R^2$ use the median heads of pinball-loss quantile models, whereas the MLP and cross-attention (X-attn) columns are point models trained with mean squared error (MSE), so their point objectives are not identical. The source-wise columns report marginal-CQR intervals for the gated source-wise model. Concat is homogeneous XGBoost on concatenated features. Best solo reports, for each dataset, the solo predictor with the largest mean $R^2$ over five seeds among the available tabular-only, text-only, and image-only models. These solos are the source-wise model's own per-source predictors evaluated alone, an XGBoost quantile triple on the tabular block and an MLP trained with the pinball loss per embedded source (\appref{app:repro}), so solo $R^2$ likewise uses median outputs.}
\label{tab:aux}
\begin{tabular*}{\linewidth}{@{\extracolsep{\fill}}lccccccc}
\toprule
Dataset & Src.-wise $R^2$ & Best solo $R^2$ & Concat $R^2$ & MLP $R^2$ & X-attn $R^2$ & Src. PICP & Src. NCIW \\
\midrule
SRED & $0.755\pm0.007$ & $0.742\pm0.007$ & $0.736\pm0.007$ & $0.817\pm0.010$ & $0.803\pm0.007$ & $0.943$ & $0.287$ \\
Mercari & $0.343\pm0.008$ & $0.275\pm0.006$ & $0.309\pm0.006$ & $0.382\pm0.014$ & $0.396\pm0.007$ & $0.948$ & $0.312$ \\
Pawpularity & $0.216\pm0.025$ & $0.240\pm0.006$ & $0.234\pm0.006$ & $0.202\pm0.017$ & $0.221\pm0.023$ & $0.953$ & $0.774$ \\
IMDB-WIKI & $0.738\pm0.006$ & $0.739\pm0.005$ & $0.678\pm0.002$ & $0.734\pm0.011$ & $0.753\pm0.002$ & $0.950$ & $0.347$ \\
\bottomrule
\end{tabular*}

\end{table}

\autoref{tab:aux} separates calibration sanity checks for the source-wise base model from descriptive point accuracy. Among the multi-modal fusion models, the neural columns trained with MSE lead on SRED, Mercari, and IMDB-WIKI, while homogeneous XGBoost leads on Pawpularity. The point objectives differ as described in the caption, and the pre-calibration data allocations differ as well, since the source-wise quantile predictors are fitted on the fitting split with the gate on the tuning split whereas the neural baselines train on the pooled fit and tuning splits, so this is neither a like-for-like optimizer comparison nor a like-for-like comparison of data budgets. The solo diagnostic also shows that Pawpularity is nearly dominated by a single source: its image-only model attains $R^2=0.240$, above the best non-solo value reported in the table, $0.234$. The source-wise model improves over homogeneous XGBoost on SRED, Mercari, and IMDB-WIKI and gives marginal-CQR intervals close to the nominal level. The wrapper itself is evaluated across base predictors and datasets in \autoref{sec:main-results}.

\FloatBarrier
\subsection{Scale Tuning at Small Sample Sizes}
\label{app:smallsample}

Sample size constrains the calibration layer before any efficiency question arises. The remarks after Proposition~\ref{prop:mondrian} give the floors for a finite interval, at least $19$ calibration scores overall and in each protected stratum at $\alpha=0.05$. Realized coverage is also only as stable as the calibration count allows. For i.i.d.\ scores from a continuous distribution, the coverage induced by a realized calibration set follows the exact $\mathrm{Beta}(m,\,n{+}1{-}m)$ law with standard deviation close to $\sqrt{\alpha(1-\alpha)/n}$, so holding realized coverage within two, one, and half a percentage point of the target at one standard deviation requires roughly $120$, $475$, and $1{,}900$ calibration examples. These are stability heuristics rather than validity thresholds.

A simulation of the scores alone isolates the data cost of tuning the scale of \autoref{eq:scale}. The calibration layer sees only pairs of scores and disagreements, so we draw $d\sim|\mathcal N(0,1)|$ and absolute residuals $e=\sigma(d)\,|\varepsilon|$ with $\varepsilon\sim\mathcal N(0,1)$ and $\sigma(d)=\sqrt{1+\gamma^{*}d^{2}}$, where $\gamma^{*}\in\{0,1,4,16\}$ controls how informative the disagreement truly is. Four rules spend the same labeled budget $N\in\{25,31,50,100,200,400,800,1600,3200\}$. The marginal rule calibrates the unscaled score on all $N$ observations. The fixed rule sets $\gamma=1$ on $d$ standardized by its known population interquartile range, a pre-specified constant, and also calibrates on all $N$. The tuned rule reserves $40\%$ of $N$, rounded to the nearest integer, for the deployed tuner of \appref{app:repro}, with its 37 candidate ratios, its interquartile standardization on the tuning sample with its fallbacks, its width proxy as the tuning objective, and its rule of keeping the smallest $\gamma$ on ties, and calibrates on the remaining $60\%$. The oracle calibrates with the conditionally valid scale $a=\sigma$. Because the noise is Gaussian, each replication's expected test coverage and width are one-dimensional integrals evaluated by dense trapezoidal quadrature, and \autoref{fig:gamma-sim} reports means over $2{,}000$ replications.

\begin{figure}[!htbp]
  \centering
  \includegraphics[width=0.98\linewidth]{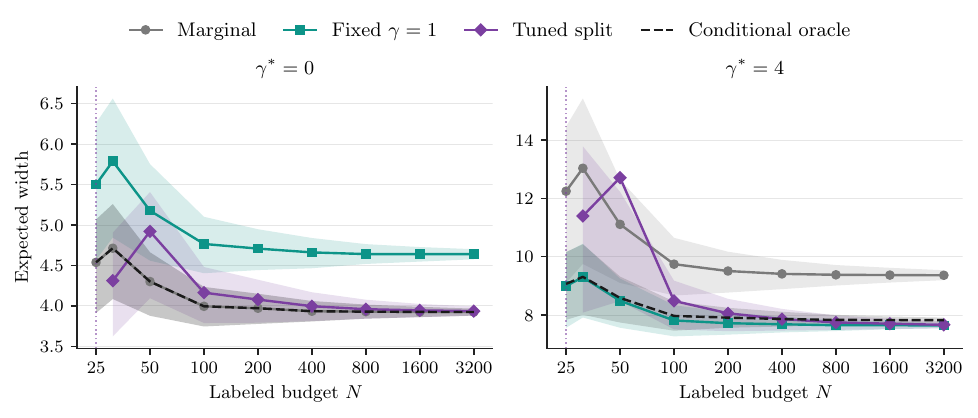}
  \caption{Expected interval width against the labeled budget $N$ in the score-level simulation, shown for two of the four simulated settings. The tuned rule runs the same tuner that the main experiments deploy, and fixed $\gamma=1$ is the pre-specified fallback for small samples. Lines are means over $2{,}000$ replications and shaded bands span the interquartile range across replications as variability from sampling the calibration scores. At the smallest finite budgets the tuned width distribution is heavy-tailed, so its mean can sit above that band. The dotted line marks $N=25$, where the tuned rule's calibration part falls below the $19$-score floor, and the tuned curves begin at $N=31$, the knife-edge budget discussed in the text. Without a signal (left), the conditional oracle coincides with the marginal rule, fixing $\gamma=1$ pays a persistent width premium, and the tuned rule approaches the marginal rule from above beyond the knife edge. With a strong signal (right), fixed $\gamma=1$ tracks the conditional oracle. The tuned rule converges to the width-optimal scale, which is flatter than the oracle, and therefore ends below the oracle at the largest budgets. It stays above the better of the marginal and fixed rules throughout, by under one percent at the largest budgets, which is the price of reserving part of the budget for tuning.}
  \label{fig:gamma-sim}
\end{figure}

Three patterns emerge. First, the floors bite exactly as stated. At $N=25$ the tuned rule's calibration part falls below $19$ scores and every tuned interval is infinite, and $N=31$ is the first budget with a finite tuned quantile under this split. At that knife edge the tuner degenerates, selecting $\gamma=0$ in every replication because twelve tuning scores cannot produce a finite conformal quantile, and the tuned rule looks narrower than the marginal rule only because its calibration count sits exactly at the floor, where the split-conformal index is least conservative, with expected coverage $19/20=0.950$ against the marginal rule's $31/32\approx0.969$, not because the scale helps. Second, splitting is costly at small budgets. At $N=50$ the tuned rule is $14\%$ to $57\%$ wider than the better of the marginal and fixed rules, and it remains $4$ to $9\%$ wider at $N=100$. Beyond the $N=31$ knife edge, tuning first beats the better of the marginal and fixed rules at about $N=400$ for $\gamma^{*}=1$ and $N=200$ for $\gamma^{*}=16$, and not within the studied range for $\gamma^{*}\in\{0,4\}$, although against the marginal rule alone it is already narrower from roughly $N=100$ whenever the signal is real. Fixing $\gamma=1$ instead encodes a prior. It stays within $6\%$ of the conditionally valid oracle's width at every budget when $\gamma^{*}\geq1$ but pays $18$ to $23\%$ over the marginal rule when $\gamma^{*}=0$. Third, the conditionally valid scale is not the width-optimal scale. Under a marginal coverage constraint, the population width-optimal members of \autoref{eq:scale} are flatter than $\sigma$, at $\gamma=0.56$, $1.81$, and $5.35$ against $\gamma^{*}=1$, $4$, and $16$, and the tuner converges to these flatter members, with median selected $\gamma$ at the largest budget of $1/3$, $4/3$, and $4$ on the tuner's standardized disagreement, about $0.50$, $1.85$, and $5.83$ after undoing each replication's interquartile standardization. Its coverage conditional on $d=3$, roughly the $99.7$th percentile of the simulated disagreement, accordingly settles near $0.90$ against the oracle's uniform $0.95$, while the marginal rule falls to $0.56$ there at $\gamma^{*}=4$. Tuning the scale for efficiency therefore buys width at the cost of coverage in the tail, and when protection for high-disagreement examples is the goal, the Mondrian construction of \autoref{sec:mondrian} is the appropriate tool. The same comparison explains the crossover pattern, since tuning overtakes the better fixed choice earliest where that choice sits far from the width-optimal member, and at $\gamma^{*}=4$ the fixed $\gamma=1$ already nearly matches it. These conclusions are specific to a design with absolute residuals, disagreement drawn as $|\mathcal N(0,1)|$, and a true scale inside the family of \autoref{eq:scale}, where the population quantities are deterministic integrals and the finite-budget curves are averages over the $2{,}000$ replications. The fixed rule is moreover given the population standardizer, which a deployment must fix in advance. The trade-off transfers to CQR scores only qualitatively.

\subsection{Two Worked Test Listings}
\label{app:worked}

\autoref{fig:examples} grounds the mechanism of \autoref{fig:mechanism} in two individual SRED test listings from the seed-0 run of the same cross-dataset configuration. For the left listing the three source predictions nearly coincide, so the disagreement-scaled interval tightens relative to marginal CQR and still covers the true rent. For the right listing the sources conflict, so the disagreement-scaled interval widens instead. These are transparent post-hoc illustrations, not additional evaluation results. Both listings were selected using test outcomes, the left among cases covered by both methods where scaling narrowed the interval and the right among cases covered by disagreement scaling but missed by marginal CQR. This selection has no role in the aggregate comparisons. The intervals are computed on log-rent and mapped monotonically back to Swiss francs. The figure's width annotations compare the displayed CHF endpoints.

\begin{figure}[!htbp]
  \centering
  \includegraphics[width=0.98\linewidth]{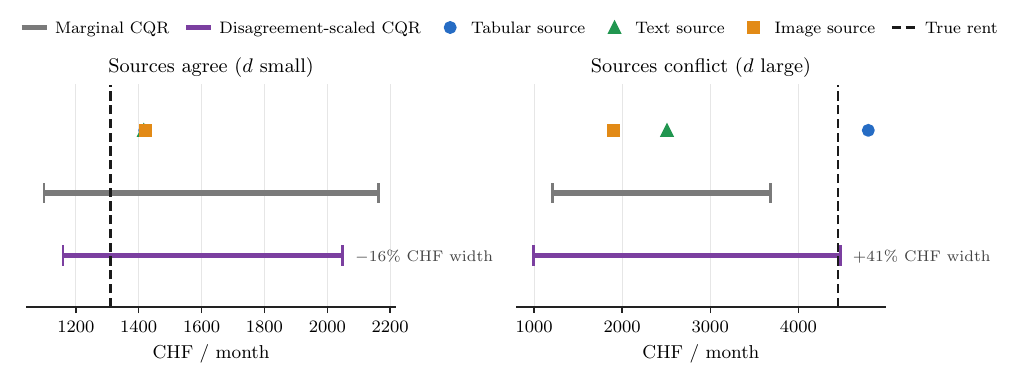}
  \caption{Two post-hoc illustrative SRED test listings from the seed-0 run under marginal and disagreement-scaled CQR. Color- and shape-coded markers show the three per-source predictions, the dashed line marks the true rent, and the bars are the calibrated intervals after mapping from log-rent to CHF. The annotations compare width in CHF. The examples were outcome-selected using the criteria stated in the text and are not aggregate performance evidence.}
  \label{fig:examples}
\end{figure}

\end{document}